\documentclass{article} 
\usepackage{iclr2026_conference,times}

\usepackage{amsmath,amsfonts,bm}
\def\methodName{\texttt{CANON}}
\def\inter{\texttt{CANON-Inter}}
\def\intra{\texttt{CANON-Intra}}
\def\eff{\texttt{CANON-Eff}}
\def\dynamic{\texttt{CANON-Dynamic}}

\def\eqref#1{equation~\ref{#1}}

\def\1{\bm{1}}

\DeclareMathAlphabet{\mathsfit}{\encodingdefault}{\sfdefault}{m}{sl}
\SetMathAlphabet{\mathsfit}{bold}{\encodingdefault}{\sfdefault}{bx}{n}

\usepackage[utf8]{inputenc} 
\usepackage[T1]{fontenc}    
\usepackage{hyperref}       
\usepackage{url}            
\usepackage{booktabs}       
\usepackage{amsfonts}       
\usepackage{nicefrac}       
\usepackage{microtype}      
\usepackage{xcolor}         
\usepackage{amsmath}
\usepackage{graphicx}

\usepackage{tcolorbox}
\usepackage{xcolor}
\usepackage{tabularx}
\tcbuselibrary{skins}
\usepackage{subcaption}
\usepackage{tocbibind}  
\usepackage{titletoc}   
\usepackage{enumitem}
\usepackage{caption}
\usepackage {wrapfig}
\usepackage{minted}
\usepackage{amsthm}
\usepackage{bm}
\newtheorem{theorem}{Theorem}

\tcbuselibrary{breakable}
\usepackage{multirow}
\usepackage{multicol}
\usepackage{colortbl}
\usepackage{fontawesome5}
\definecolor{my_green}{RGB}{51,102,0}
\definecolor{my_purple}{RGB}{160, 43, 147}
\definecolor{my_blue}{RGB}{15, 158, 213}
\definecolor{darkpink}{RGB}{255, 20, 147}
\definecolor{high}{HTML}{6AD4DD}
\definecolor{math}{HTML}{FFA09B}

\NewDocumentCommand{\yafu}
{ mO{} }{\textcolor{red}{\textsuperscript{\textit{yafu}}\textsf{\textbf{\small[#1]}}}}

\NewDocumentCommand{\fixed}
{ mO{} }{\textcolor{blue}{\textsuperscript{\textit{yafu}}\textsf{\textbf{\small[#1]}}}}

\title{Conditional Advantage Estimation for \\Reinforcement Learning in Large Rea-\\soning Models}

\iclrfinalcopy

\author{
\textbf{Guanxu Chen}$^{1,2}\thanks{This work was done during an internship at Shanghai Artificial Intelligence Laboratory, supervised by Dongrui Liu. Our code is available at \textcolor{darkpink}{\faGithub \href{https://github.com/biuboomc/CANON}{\methodName{}}}.}$\quad
\textbf{Yafu Li}$^{2}$$^\dagger$\quad
\textbf{Yuxian Jiang}$^{3}$\quad
\textbf{Chen Qian}$^{4}$\quad
\textbf{Qihan Ren}$^{1}$\quad
\\ ~\textbf{JingYi Yang}$^{3,2}$\quad
\textbf{Yu Cheng}$^{5}$\quad
\textbf{Dongrui Liu}$^{2}$$^\dagger$\quad
\textbf{Jing Shao}$^{2}\thanks{Corresponding Author.}$\quad
\vspace{1em}\\
$^1$ Shanghai Jiao Tong University, 
$^2$ Shanghai Artificial Intelligence Laboratory,\\
$^3$ Fudan University, 
$^4$ Renmin University of China, 
$^5$ The Chinese University of Hong Kong\\
\it\footnotesize ~~lm.cgx@sjtu.edu.cn\quad \quad yafuly@gmail.com\quad\quad\{liudongrui, shaojing\}@pjlab.org.cn
}

\begin{document}

\maketitle

\begin{abstract}
Reinforcement Learning with Verifiable Rewards (RLVR) for large language models (LLMs) has achieved remarkable progress in enhancing LLMs’ reasoning capabilities on tasks with clear correctness criteria, such as mathematical reasoning tasks. Several training metrics, such as entropy or response length, have been observed to correlate with different reasoning behaviors in reinforcement learning. Prior approaches incorporate such priors through reward or advantage shaping, which often relies on hand-crafted penalties and preferences (e.g., higher-is-better or lower-is-better). However, without careful hyper-parameter tuning, these directional priors can be overly biased and may lead to failure. To this end, we introduce \textit{\textbf{C}onditional adv\textbf{AN}tage estimati\textbf{ON}} (\methodName{}), amplifying the impact of the target metric without presuming its direction. Specifically, \methodName{} regroups the sampled responses into two groups based on the higher or lower value of a target metric, measures which metric trend contributes to better performance through inter-group comparison, and identifies the better response within the same group. In summary, \methodName{} based on entropy consistently outperforms prior methods across three LLMs on both math reasoning and high-complexity logic tasks. When applied to response length, \methodName{} further improves token efficiency, yielding a more favorable Pareto frontier in the performance–cost trade-off.
\end{abstract}
\begin{figure}[h]
    \vspace{-15pt}
    \centering
    \includegraphics[width=1.0\linewidth]{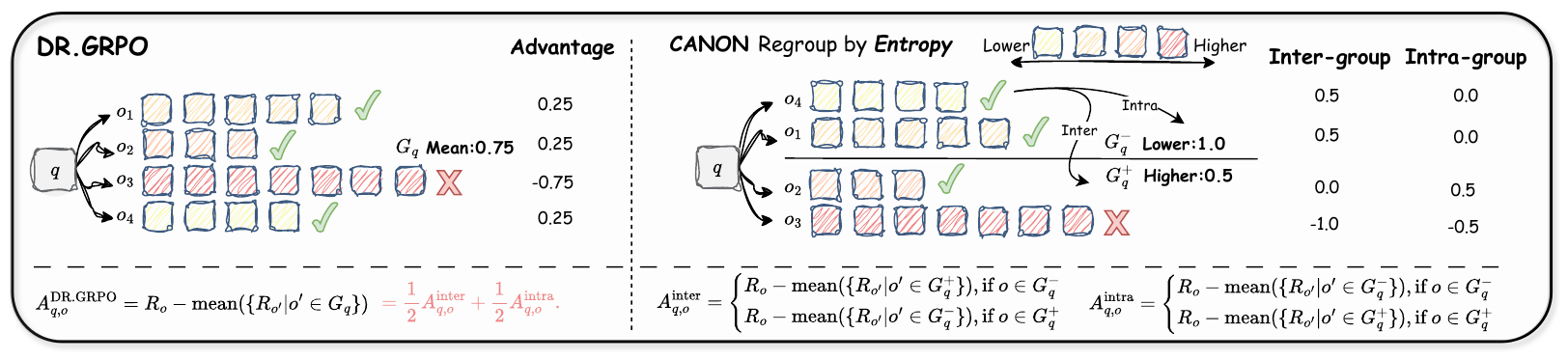}
    \vspace{-15pt}
    \caption{\small
    \methodName{} amplifies the impact of specific metric changes by regrouping the sampled responses into two groups based on the values of a given metric. Rather than comparing against the mean value of all responses like DR.GRPO, \methodName{} selects the direction of metric change that offers greater contributions to performance through inter-group comparison and favors responses that exhibit better performance within groups following the same trend in its intra-group comparison. DR.GRPO can be expressed as the average of \methodName{}’s two advantage estimates and is therefore a special case of \methodName{}.
    }
    \vspace{-15pt}
    \label{fig:INTRO}
\end{figure}
\section{Introduction}
Recently, Large Reasoning Models (LRMs) such as Gemini 2.5 Pro \citep{gemini}, DeepSeek-R1 \citep{guo2025deepseek}, and OpenAI-o1 \citep{o1_card}, continue to push the boundaries of performance on reasoning tasks. A key technique driving this success is Reinforcement Learning with Verifiable Rewards (RLVR), which enables models to refine answers through multi-step reflection \citep{zhang2025survey}. Algorithms designed for RLVR, most prominently GRPO \citep{shao2024deepseekmath} and its variants (e.g., DR.GRPO, \cite{liu2025understanding}), have become central to achieving superior performance.

In previous works, some training metrics are observed to be closely correlated with model behavior, which can guide the training process and improve LLMs' performance \citep{hassid2025don,gandhi2025cognitive,wang2025beyond}. To incorporate such a human prior, some methods integrate these metrics through reward shaping \citep{cmu1,o1} and advantage shaping \citep{chen2025seed,cheng2025reasoning} to guide the model's reasoning behavior. For example, an over-length penalty is used to boost reasoning efficiency, and the entropy signal is leveraged to maintain exploration for better performance. 

However, these methods usually introduce human priors by adding penalty and reward terms, which hold handcrafted priors that specific metrics are either to be higher-is-better or to be lower-is-better. 
Without careful hyper-parameter selection, these priors can be overly biased and drive specific metrics up or down directly, thus failing to enhance performance robustly. Simple handcrafted priors towards one specific direction are hard to work in different scenarios. For instance, higher-entropy responses tend to be exploratory and may correctly answer complex questions, whereas lower-entropy responses exhibit higher certainty and achieve greater accuracy on most questions within their capability \citep{cheng2025reasoning,prabhudesai2025maximizing,wang2025beyond}. Therefore, we aim to amplify the impact of specific metric changes without presupposing preferences, naturally identifying inherent tendencies in model rollouts that can be leveraged to facilitate learning of beneficial behaviors, such as enhancing exploration or improving reasoning efficiency.

To this end, we regroup the sampled responses into two groups based on the higher or lower values of a given metric during the process of RLVR training. 
Specifically, we can sort the sampled responses according to the value and split them into two groups. Based on this, we propose \textit{\textbf{C}onditional adv\textbf{AN}tage estimati\textbf{ON}} (\methodName{}), which computes the inter-group advantage by comparing a response with the group that it does not belong to, and gets the intra-group advantage across its own group conversely. The inter-group advantage reveals which trend of metrics leads to higher accuracy. Meanwhile, the intra-group advantage identifies better responses within the same group. 

Taking the metric of entropy as an example, if groups with lower entropy (i.e., higher certainty) yield higher average rewards, the inter-group advantage tends to select correct responses with low entropy, efficiently exploiting existing features to boost performance. In contrast, correct rollouts with higher entropy receive more advantages in the intra-group comparison because the average reward of their group is lower, thereby encouraging truly effective exploration. We theoretically prove that when the two groups have equal size, the inter-group advantage amplifies the impact of the grouping metric on the advantage computation. In this setting, DR.GRPO can be formulated as a uniform weighting of these two advantages, which is a special case of \methodName{}.

We consider the metrics of generation entropy and response length, evaluating the effectiveness of \methodName{} on three open-weight LLMs across six math reasoning benchmarks and three challenging logic reasoning tasks. 
Empirical results show that emphasizing the inter-group advantage based on entropy yields a \textbf{1.9}-point accuracy gain on math tasks. In contrast, for high-complexity reasoning problems, the intra-group advantage proves crucial, achieving a \textbf{5.2}-point improvement on the most challenging subset. Through scheduling of these advantages, \methodName{} further achieves a superior and comprehensive performance across three models and two tasks. Furthermore, \methodName{} based on response length substantially enhances reasoning efficiency, establishing a new Pareto frontier in the performance–efficiency trade-off. In low-token-budget scenarios for math tasks, it achieves \textbf{2.63×} higher performance and reduces token consumption by \textbf{45.5\%} at the same performance level.
\vspace{-15pt}
\section{Related Work}
\vspace{-5pt}
\textbf{Advantage Estimations in Reinforcement Learning.} In PPO, the advantage estimation is provided by Generalized Advantage Estimation (GAE, \cite{schulman2015high}).To avoid the computational cost of the critic model, several methods, such as ReMax \citep{li2023remax}, RLOO \citep{ahmadian2024back}, GRPO \cite{shao2024deepseekmath}, and REINFORCE++ \citep{hu2025reinforce++}, utilize alternative techniques like baseline reward and group-relative rewards for advantage estimation. ReMax compares the rewards with the baseline reward from the greedy decoding response. REINFORCE++ estimates the advantage by the normalization operation across the global batch for all queries. RLOO and GRPO estimate the advantage in a group relative manner. RLOO computes the average rewards of all other solutions in the group as the baseline reward, and GRPO utilizes the normalized rewards among the sampled solutions as the advantage estimation. Compared to GRPO, our method splits sampled responses into two groups based on specific conditions and selects the appropriate condition through inter- and intra-group comparisons, thereby efficiently optimizing key patterns that boost task performance.

\textbf{Reinforcement Learning with Verifiable Rewards.} RLVR leverages the existing RLHF objective \citep{ppo} but replaces the reward model with a verification function, which is available in domains with verifiable answers, such as mathematics reasoning tasks \citep{guo2025deepseek,lambert2024tulu}. \cite{yu2025dapo,liu2025cpgd,chen2025minimax} consider the importance sampling techniques and contribute novel training paradigms and optimization objectives for better and more stable reasoning capabilities. Due to the sparse rewards during training, past methods utilize not only accuracy-based rewards but also explicitly integrate additional signals through reward shaping \citep{cmu1,o1} and advantage shaping \citep{chen2025seed,cheng2025reasoning} to guide the model's reasoning and reflection. \cite{cmu1} and \cite{o1} utilize an over-length penalty to boost reasoning efficiency. \cite{chen2025seed} and \citep{cheng2025reasoning} consider the entropy as a measure of exploration and reshape the advantage computation. \cite{gandhi2025cognitive} also observes four key cognitive behaviors of initial reasoning behaviors and strengthens the capacity for self-improvement. However, these methods usually introduce human priors by adding penalty and reward terms, which hold handcrafted priors that can be overly biased and may fail to enhance performance without careful hyper-parameter selection. Our work amplifies the impact of specific metric changes without presupposing preferences, leveraging them to facilitate learning of beneficial behaviors.

\vspace{-10pt}
\section{Preliminaries}
\vspace{-5pt}
Proximal Policy Optimization (PPO, \cite{ppo}) is a widely used method for policy optimization of LLMs. PPO utilizes the clip mechanism to update policy stably. PPO maximizes the following optimization objectives.
\begin{align}
\mathcal{J}_\text{PPO}(\theta) = \mathbb{E}_{q\sim \mathcal{D},o\sim\pi_{\theta_{\text{old}}}(\cdot\mid q)}
\Bigg[ \frac{1}{|o|}\sum_{t=1}^{|o|}
\min \Bigg( r_{o,t}(\theta) \hat{A}_t,  
\ \text{clip}_{1 - \varepsilon}^{1 + \varepsilon} (r_{o_i,t}(\theta)) \hat{A}_t \Bigg) \Bigg]~,
\end{align}
where $\pi_{\theta_{\text{old}}}$ and $\pi_{\theta}$ are used to denote the policy model before and after the update. $q$ is a query sampled from the data distribution $\mathcal{D}$, and the output $o$ is generated by $\pi_{\theta_{\text{old}}}$. The clipping function with clip ratio $\varepsilon$ is computed as $\text{clip}_a^b(x)=\max(\min(x,a),b)$ and the importance sampling ratio at time step $t$ is defined as $r_{o,t}(\theta) = \frac{\pi_{\theta}(o_t\mid q,o_{<t})}{\pi_{\theta_{\text{old}}}(o_t\mid q,o_{<t})}$.

To avoid the computational cost of the critic model, GRPO \citep{shao2024deepseekmath} estimates the advantage in a group relative manner. They sample $G$ different solutions for the current query $q$ as the group $G_q:=\{o|o\sim\pi_{\theta_{\text{old}}}(.|q)\}$, and calculate the normalized rewards as advantages within the group $G_q$.
\begin{align}
    \hat{A}^{\text{GRPO}}_{q,o,t}  =  \frac{R_{o} - \text{mean}(\{R_{o^\prime }|o^\prime \in G_q\})}{\text{std}(\{R_{o^\prime }|o^\prime \in G_q\})}.
\end{align}

Due to the success of DeepSeek-R1, several studies have proposed improvements based on GRPO.
DR.GRPO \citep{liu2025understanding} uses the GRPO advantages without standard deviation normalization 
and develops a token-level loss without length bias.

\begin{figure}[t]
    \centering
    \includegraphics[width=0.95\linewidth]{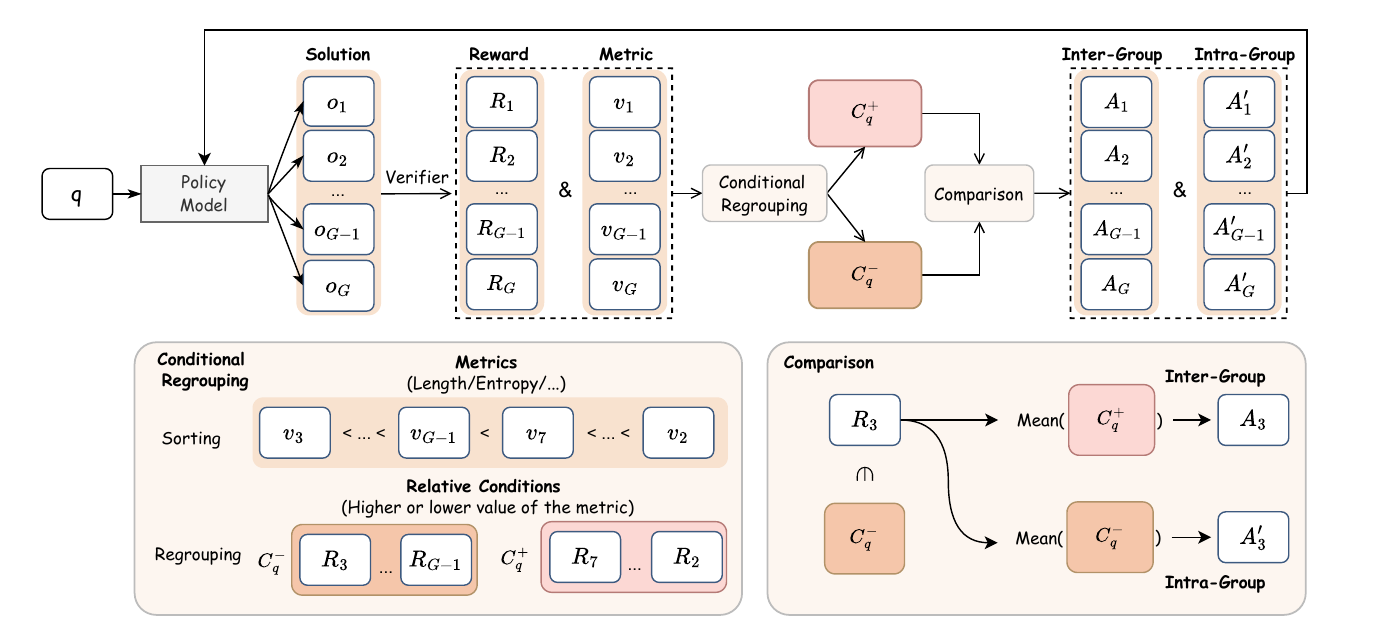}
    \caption{
    Overview of \methodName{}. \methodName{} regroups all the sampled responses based on the value of a specific metric, and computes the advantages through inter-group and intra-group comparison.
    }
    \vspace{-15pt}
    \label{fig:CG}
\end{figure}
\vspace{-1pt}
\section{Conditional Advantage Estimation} 
\vspace{-5pt}
\label{sec:method}
Group-based advantage estimation methods, such as GRPO, typically use the average reward of all sampled responses within the group as a baseline reward. This may fail to provide a clear feedback signal for policy optimization due to the ambiguity of the comparison target. We propose \methodName{}, which performs conditional regrouping by splitting all sampled responses into two groups based on the value of a specific metric. Leveraging these two groups, inter-group advantage identifies the metric trend that yields higher accuracy through cross-group comparison, while intra-group advantage selects superior responses within the same trend and prioritizes correct answers from the group with a lower average reward. 


\subsection{Conditional Regrouping}
To explicitly introduce a comparison target, we regroup all the sampled responses based on specific conditions. Given any condition $c$, we denote the set of all outputs for the current query $q$ that satisfy this condition in the sampled group $G_q$ as $C^+_q:=\{o|o~\text{satisfy}~c,~o\in G_q\}$. The set of outputs that do not satisfy the condition can be denoted by $C^-_q=G_q\setminus C_q$. In this work, we focus on studying the relative conditions given by the training metrics, such as the entropy and length of the sampled responses. Specifically, we divide the responses into two non-overlapping groups based on the value of the metrics, as shown in Figure \ref{fig:CG}.

\subsection{Advantage Estimation Based On Regrouping.}
\label{sec:cae}
Given two groups, we can compute the inter-group advantage through comparison between different groups.
\begin{align}
\hat{A}^{\text{inter}}_{q,o,t}  &=\left\{
                \begin{array}{ll}
                 R_{o} - \text{mean}(\{R_{o^\prime}|o^\prime \in G_q^+\}), \text{if} ~ o \in G_q^-\\
                  \\
                R_{o} - \text{mean}(\{R_{o^\prime}|o^\prime \in G_q^-\}), \text{if} ~ o \in G_q^+\\
                \end{array}\right. .
\end{align}
Meanwhile, we also compute the intra-group advantage by comparing each response with the mean reward of its own group. 
\begin{align}
\hat{A}^{\text{intra}}_{q,o,t}  &=\left\{
                \begin{array}{ll}
                 R_{o} - \text{mean}(\{R_{o^\prime}|o^\prime \in G_q^+\}), \text{if} ~ o \in G_q^+\\
                  \\
                R_{o} - \text{mean}(\{R_{o^\prime}|o^\prime \in G_q^-\}), \text{if} ~ o \in G_q^-\\
                \end{array}\right. .
\end{align}
Although this may appear similar to the estimation of DR.GRPO within a smaller scope, due to the differing average advantages between groups, the intra-group advantage prioritizes correct responses from the group with a lower average reward ($1 - \text{mean}(\{R_{o^\prime}|o^\prime \in G_q^+\} > 1 - \text{mean}(\{R_{o^\prime}|o^\prime \in G_q^-\} \text{~when~} \text{mean}(\{R_{o^\prime}|o^\prime \in G_q^+\} < \text{mean}(\{R_{o^\prime}|o^\prime \in G_q^-\}$   ). We can further combine the above two advantages into a unified formulation.
\begin{align}
\hat{A}_{q,o,t}^{\methodName{}} & = \mu \hat{A}^{\text{inter}}_{q,o,t} + (1-\mu)\hat{A}^{\text{intra}}_{q,o,t},
\label{eq:canon}
\end{align}
where $\mu$ controls the balance between the inter-group and intra-group advantage. Figure \ref{fig:CG} demonstrates a concise case of the computation of \methodName{}.

To ensure that the advantages introduced by conditional regrouping provide a clearer contrastive signal, we theoretically analyze the situations under which inter-group advantage, compared to DR.GRPO, yields a stronger advantage signal in response to reward gaps under specific conditions.
\begin{theorem}[Situations with clearer advantage signal (proved in Appendix \ref{derivation})]
Suppose that condition c is based on numerical comparisons and can be derived through sorting of metrics. Further assume that the sampled response $o$ to query $q$ satisfy condition c with probability $p \in (0,1)$, and $\mathbf{E}_{o \text{~satisfy~}c}[R_o] \neq \mathbf{E}_{o \text{~not ~satisfy~}c}[R_o]$. Then, we have:
\begin{align}
\frac{|\hat{A}^{\text{inter}}_{q,o,t}|}{|\hat{A}^{\text{DR.GRPO}}_{q,o,t}|} > 1 , \text{ only when }|C^+_q|=|C^-_q| \text{ if  }|C^+_q|\text{ is a constant.  }
\end{align}
\label{thm1}
\end{theorem}

Based on Theorem \ref{thm1}, we divide the responses into two equally sized groups. In this way, DR.GRPO can be expressed as a special case of this unified form when $\mu=0.5$.
\begin{align}
\hat{A}^{\text{DR.GRPO}}_{q,o,t}  &=  R_{o} - \text{mean}(\{R_{o^\prime }|o^\prime \in G_q\}) = \frac{1}{2} \hat{A}^{\text{inter}}_{q,o,t} + \frac{1}{2}\hat{A}^{\text{intra}}_{q,o,t}.
\end{align}

\subsection{Aligning with Training Target through Weighted Advantage}
According to Section \ref{sec:cae}, the selection between different trends of metrics only takes place in the inter-group advantage. By weighting different conditions within the inter-group advantage calculation, this enables fine-grained control over the trend of metrics with only tiny differences compared to DR.GRPO. For instance, by slightly reducing the weight of longer responses, \methodName{} can accomplish reasoning of high token efficiency through the RL process. Specifically, the inter-group advantage in the Eq. \ref{eq:canon} should be replaced with $\hat{A}^{\text{inter}}_{q,o,t, \alpha}$
where $\alpha$ is the weight of a specific group, and $\hat{A}^{\text{inter}}_{q,o,t, \alpha}$ is defined as:
\begin{align}
\hat{A}^{\text{inter}}_{q,o,t, \alpha}  &=\left\{
                \begin{array}{ll}
                  R_{o} - \alpha * \text{mean}(\{R_{o^\prime}|o^\prime \in G_q^+\}), \text{if} ~ o \in G_q^-\\
                  \\
                \alpha * R_{o} - \text{mean}(\{R_{o^\prime}|o^\prime \in G_q^-\}), \text{if} ~ o \in G_q^+\\
                \end{array}\right. .
\label{eq:weight}
\end{align}
For example, setting $\alpha$ as 0.9 can achieve substantial length reduction with little performance drop, where $C_q^+$ is considered the group with longer responses.
\section{Experiments}
\label{sec:exp}
The empirical evaluation of \methodName{} consists of three parts. Firstly, we demonstrate the effect of intra-group and inter-group advantages, respectively, across six math reasoning benchmarks and one high-complexity logic reasoning benchmark. In the second part, we perform several scheduling tricks to get the frontier in both tasks. At last, by weighting the longer responses with $\alpha<1$, we achieve efficient reasoning that reaches a better Pareto frontier.
\subsection{Performance of Intra-group and Inter-group Advantages.}
\label{sec:exp1}
\textbf{Training Setup.} We select the response length and the per-token generation entropy, respectively, to regroup the sampled solutions. We use a subset with 45k prompts from OpenR1-Math-220k \citep{openr1} that is filtered and constructed by \citet{yan2025learning}. Following DR.GRPO \citep{liu2025understanding} and DAPO \citep{yu2025dapo}, we correct the response-level length bias and utilize the clip-higher strategy ($\epsilon_{high} = 0.28$) for all experiments. We also remove both the KL loss and the entropy loss. We sample 16 responses per prompt and use temperature=1.0 for rollout generation. Our rollout batch size is 512, and the train batch size is 32. The responses to the same prompt are separated into two evenly sized groups by sorting ordinal variables. We conduct the main experiments on Qwen2.5-Math-7B \citep{qwen2.5_math} following \cite{zeng2025simplerl,liu2025understanding,yan2025learning}. We expand Qwen2.5-Math-7B's context limit from 4096 to 16384 by changing the rope theta from 10000 to 40000\footnote{The original context limit leads to unacceptable length clipping ratio. Please see Figure \ref{fig:clip_length} in Appendix \ref{app:sch3}.}. We set the maximum answer length to 8192 and the learning rate is set to 1e-6. We use \textit{Math-Verify} to give the 0-1 score for both training reward and evaluation accuracy.

\textbf{Evaluation Setup.} We evaluate the math reasoning capabilities on six commonly used benchmarks, such as MATH-500 \citep{dataset_math}, GSM8K \citep{cobbe2021training}, AMC \citep{li2024numinamath}, OlympiadBench \citep{dataset_olympiad}, and AIME 24/25. Due to the tiny size of AIME 24/25 and AMC, we report \textit{Avg@10} as the test accuracy. For the other benchmarks, we compute the \textit{Pass@1} as the test performance. We calculate the average performance and token cost across all benchmarks. All models are evaluated under the same setting with a temperature of 0.6. The values in Table \ref{tab:main} are the percentage accuracy of the models evaluated. We also select three high-complexity subsets of ZebraLogic \citep{zebra} with their solution space sizes greater than $10^3$ (Mid), $10^6$ (Large), and $10^9$ (XLarge), respectively. In this experiment, we record six metrics, including training reward, generation entropy, response length, the test performance of math tasks and logic reasoning task, and the marginal improvement gained from reflection.

\begin{table*}[t]
\centering

\caption{Overall performance based on \textbf{Qwen2.5-Math-7B}.
We compare with the following baselines: (1) Qwen2.5-Math-7B-Instruct (Qwen-Instruct), (2) prior advantage estimation methods. 
All models are evaluated under a unified setting.  
\textbf{Bold} and \underline{underline} indicate the best and second-best results, respectively.}
\label{tab:merged_results}
\setlength{\tabcolsep}{2.5pt}  
\renewcommand{\arraystretch}{1.3} 
\resizebox{\textwidth}{!}{%
\begin{tabular}{lcccccc|c>{\columncolor{math!20}}c|ccc|c>{\columncolor{high!20}}c}
\toprule
\multirow{2}{*}{\textbf{Model}} & \multicolumn{8}{c}{\textbf{Math Reasoning}} & \multicolumn{5}{c}{\textbf{High Complexity Reasoning}}  \\
\cmidrule(lr){2-9} \cmidrule(lr){10-14}
 & $\textbf{AIME 24}$ & $\textbf{AIME 25}$ & $\textbf{Olympiad}$ & $\textbf{AMC}$  & $\textbf{MATH-500}$ & $\textbf{GSM8k}$ & \textbf{Tokens}  & \textbf{Acc} & $\textbf{Mid}$  & $\textbf{Large}$ & $\textbf{XLarge}$ & \textbf{Tokens}  & \textbf{Acc}\\
\midrule
Base 
   & 16.0 &	8.0 &	26.4 &	41.6 &	61.2 &	61.6 &	2046 	&35.8 	&0.0 &	0.5 &	0.1 &	3303 &	0.2  \\
Instruct
   & 10.7 &	9.7 &	39.7 &	49.3 	&82.2 &	94.8 	&1077 &	47.7 &	11.6 &	6.2 &	3.5 &	2647 &	7.1     \\
\rowcolor{gray!30}
\multicolumn{14}{c}{Previous Advantage Estimation} \\
ReMax    & 23.3 &	18.0 &	48.1 	&62.8 &	83.4 &	90.3 &	2418 &	54.3 &	37.2& 	21.0 &	9.7 &	6246 	&22.6 \\
R++   & 20.3 &	\underline{19.7} &	45.8 &	58.3 &	82.6 &	90.0 &	4107 &	52.8 &	33.8 &	11.9 &	3.3 &	9923 &	16.3  \\
RLOO & 25.0 &	18.7 	&\underline{51.3} &	\textbf{64.3} 	&84.0 &	91.0 &	2537 &	55.7 &	33.9 &	14.4& 	5.8 &	10610 &	18.0   \\
GRPO & 22.3 &	18.3 	&47.3 &	60.6 &	83.8 &	90.8 &	3730 &	53.8 &	31.5 &	14.9 &	5.2 &	9406 &	17.2 \\
DR.GRPO ($\mu = 0.5$)  &\underline{27.7} &	\textbf{20.3} &	48.4 	&63.4 &	83.2 &	91.1 &	1522 &	\underline{55.7} &	39.2 &	24.4 &	15.1 &	4896 &	26.2  \\
\rowcolor{gray!30}
\multicolumn{14}{c}{Our Methods (Conditional Groups based on \textit{Length})} \\	 
\intra{}  & 21.7 &	19.0 &	49.9 &	63.0 &	86.2& 	\textbf{92.2} &	2176 &	55.3 &	\underline{41.8} &	25.6 &	14.7 &	4364 &	27.4 \\
\inter{}  & 27.3 &	19.3 &	47.6 &	64.2 	&82.6 	&91.0 &	\textbf{1008} 	&55.3 	&\textbf{42.7} &	\textbf{28.6} 	&17.1 &	\underline{3652} 	&\textbf{29.5}  \\
\rowcolor{gray!30}
\multicolumn{14}{c}{Our Methods (Conditional Groups based on \textit{Entropy})} \\
\intra{} & 25.0 &	16.0 &	48.9 	&62.7 &	\underline{84.4} 	&91.1 &	2959 &	54.7 &	39.1 	&\underline{27.8} &	\textbf{20.3} &	\textbf{3101} 	&\underline{29.1}  \\
\inter{} & \textbf{32.7} & 	18.7 & 	\textbf{51.7} & 	\underline{64.2} & 	\textbf{87.0} & 	\underline{91.1} & 	\underline{1466} & 	\textbf{57.6} 	& 36.3 & 	25.8 & 	14.9 & 	4415 & 	25.7  \\
\bottomrule
\end{tabular}%
\label{tab:main}
}
    \vspace{-15pt}

\end{table*}

\textbf{Baselines.} In this subsection, we fix $\alpha=1.0$ in Eq. \ref{eq:weight} and present the results of $\mu=0.0$ (\intra{}) and $\mu=1.0$ (\inter{}) in Eq. \ref{eq:canon}. A more detailed scheduling on $\mu$ will be conducted in Section \ref{sec:exp2}, and the adjustment of $\alpha$ will be covered in Section \ref{sec:exp3}. We compare \methodName{} with two types of baselines: 
(1) \textbf{Qwen2.5-Math-7B-Instruct} (Instruct, \cite{qwen2.5_math}), and (2) \textbf{previous advantage estimation methods}, such as ReMax, REINFORCE++ (R++), RLOO, GRPO, and DR.GRPO.

\textbf{Inter-group advantage achieves higher accuracy and lower length in math tasks.} The experimental results are shown in Table \ref{tab:main}. \inter{} based on \textit{Entropy} achieves an average performance of 57.6 among six math benchmarks, which is 1.9 points higher than the DR.GRPO (55.7). Specifically, \inter{} based on \textit{Entropy} has the best performance on four of the six benchmarks, and is highly competitive with the top-performing models on the rest. In AIME24, the model's performance is 5.0 points higher than the DR.GRPO's. Meanwhile, \inter{} based on \textit{Length} reduces the token cost by 33.8\% compared with DR.GRPO, while maintaining nearly unchanged performance (55.7 vs. 55.3).

\textbf{The benefit of intra-group advantage grows as the logic reasoning task's complexity increases.} Table \ref{tab:main} demonstrates that \intra{} based on \textit{Entropy} achieves higher performance of 2.9 points and 36.6\% shorter length compared with DR.GRPO. Its performance edge over DR.GRPO increases (from -0.1 to 3.4 and then 5.2) when the complexity becomes higher. The results of \intra{} based on \textit{Length} shows another trend, whose inter-group advantage makes the best performance in this task. 

\textbf{Training dynamics reflect different roles of CANON-Intra and CANON-Inter.} To be specific, we record training curves under the setting of \methodName{} based on \textit{Entropy}. The training dynamic shown in Figure \ref{fig:metrics} indicates that both the training reward and the test performance of the math tasks increase rapidly when only \inter{} is utilized ($\mu=1.0$). Its generation entropy stably decreases, and the response length changes smoothly. When using only \intra{} ($\mu=0.0$), the responses show a greater tendency for exploration. We divide the responses into two groups by counting reflection patterns and calculate the gap in average reward between the group with more and fewer reflections (Figure 2f). Figure \ref{fig:metrics} demonstrates that the trend of high-complexity reasoning performance is highly consistent with the curve of reflection gains. In the later stages of training (after approximately 90 steps), the reflection gain curve of intra-group advantage increases and finally crosses the zero point. At the same time, its performance experiences rapid growth, significantly outperforming the other two advantages.

\begin{figure}[t]
    \centering
    \includegraphics[width=0.95\linewidth]{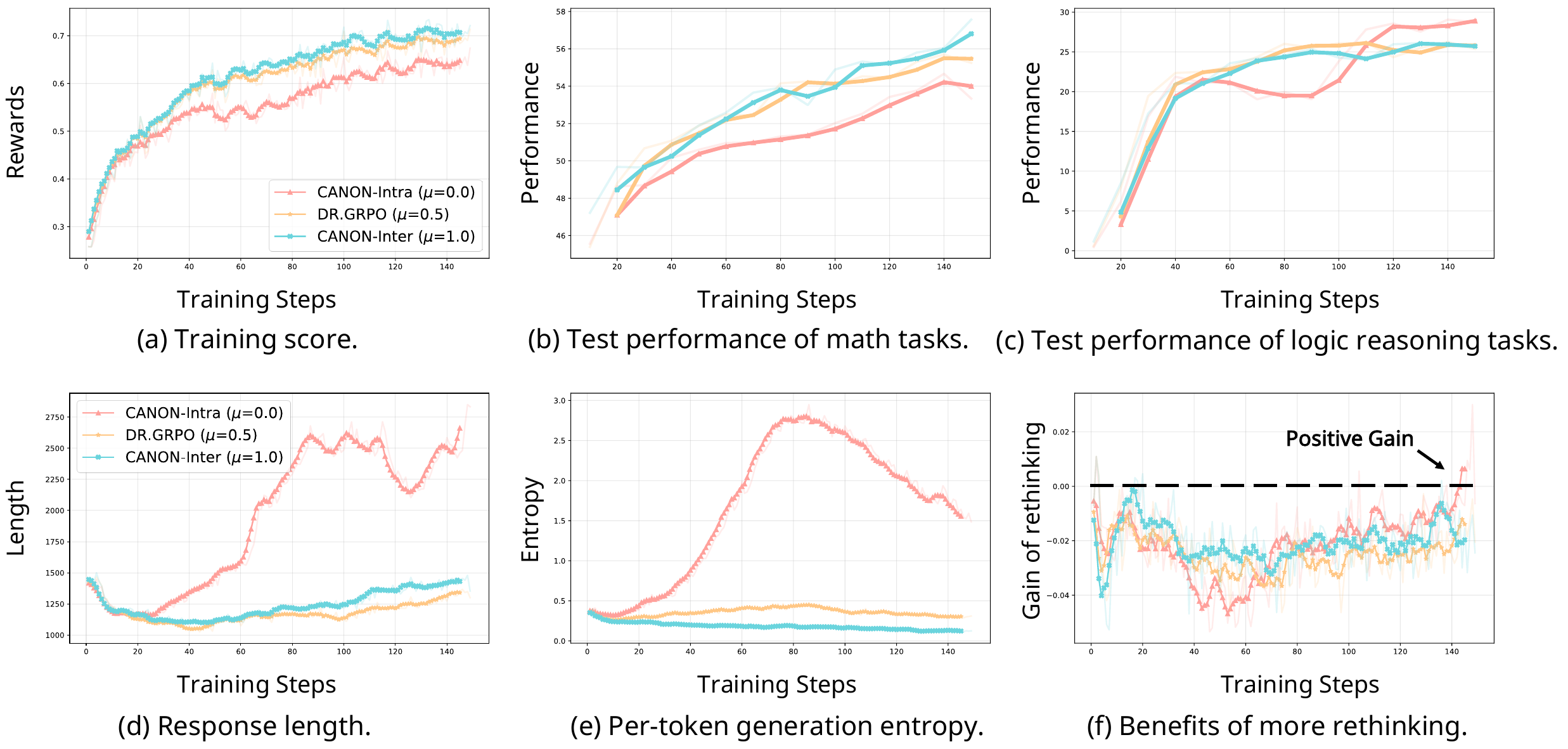}
    \caption{The training dynamics and average test performance of \inter{}, DR.GRPO, and \intra{}.}
    \label{fig:metrics}
\vspace{-20pt}
    
\end{figure}


\subsection{Balancing Performance via Advantage Scheduling}

\label{sec:exp2}
As shown in Table \ref{tab:main} and Figure \ref{fig:metrics}, \inter{} and \intra{} outperform DR.GRPO on the math reasoning task and the complex logic reasoning task, respectively, but neither can achieve the best performance on both simultaneously. To this end, we schedule the \inter{} and \intra{} by leveraging accuracy and the training steps to achieve a better balance between the two scenarios. 

\textbf{Setup.}
We conduct experiments across six math benchmarks and three complex logic reasoning tasks on Qwen2.5-Math-7B \citep{qwen2.5_math}, Llama3.1-8B \citep{dubey2024llama}, and Qwen2.5-Math-1.5B \citep{qwen2.5_math}.  
For the two Qwen series models, we use the dataset introduced in Section \ref{sec:exp1}. Due to the weak capability of Llama3.1-8B, we collect a simpler dataset with 35k samples from four open-source datasets and follow the other training setups described in Section \ref{sec:exp1}. Please see the details of this newly constructed dataset in Appendix \ref{app:sch5}. We draw a radar chart with the average performance of the two scenarios for visualization, and the results for \methodName{} with scheduling are denoted as \dynamic{}.

\textbf{Scheduling strategies.} All of the strategies are based on the coefficient $\mu$ in the Eq. \ref{eq:canon}, which balances the \inter{} and \intra{}. We try four scheduling strategies utilizing the training accuracy and training steps, respectively: (1) \textit{First-Inter-Later-Intra}. We set the value of $\mu$ to $1 - \Lambda$, where $\Lambda$ denotes the mean accuracy of current whole batch; (2) \textit{First-Intra-Later-Inter}. We set the value of $\mu$ to $\Lambda$. (3) \textit{Cosin-First-Inter-Later-Intra}. We schedule the value of $\mu$ from high to low using a cosine annealing function with restarts and warm-up. (4) \textit{Cosin-First-Intra-Later-Inter}. We schedule the value of $\mu$ from low to high using a cosine annealing function with restarts and warm-up. Please see Appendix \ref{app:sch6} for more details. The shown results of \dynamic{} are derived from one of the tried scheduling strategies that achieve strong performance in both scenarios. Ultimately, based on training performance, we select strategy \textit{Cosin-First-Inter-Later-Intra} for Qwen2.5-Math-7B and Llama3.1-8B, and strategy \textit{First-Inter-Later-Intra} for Qwen2.5-Math-1.5B.

\begin{wrapfigure}{r}{0.45\textwidth}
  \centering
  \vspace{-15pt} 
  \includegraphics[width=\linewidth]{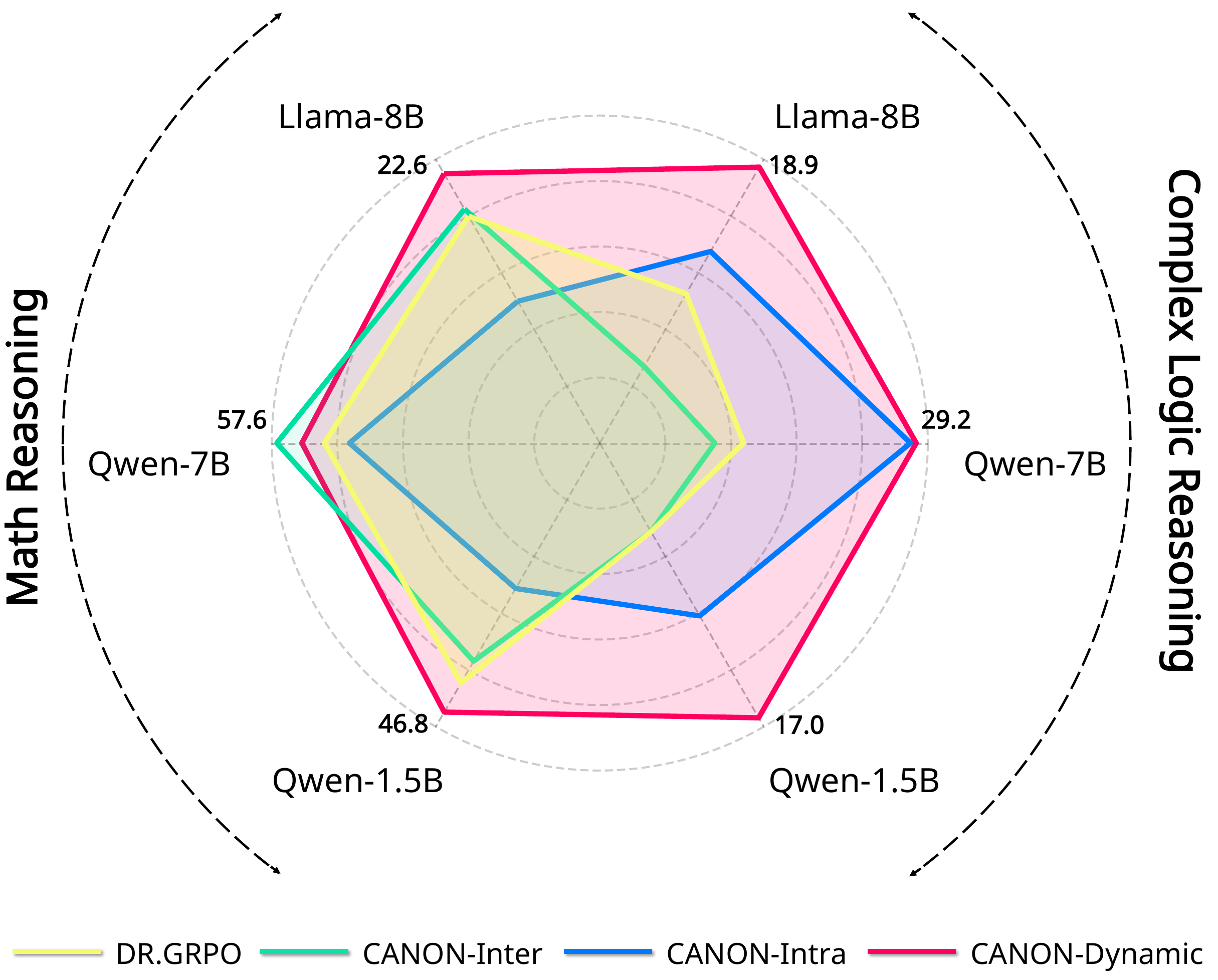}
  \vspace{-10pt}
  \caption{\footnotesize Evaluation for three LLMs across two types of reasoning tasks. \dynamic{} with scheduled advantages significantly outperforms DR.GRPO in almost all the evaluations.}
  \label{new_exp2_2}
  \vspace{-20pt}
\end{wrapfigure}

\textbf{\methodName{} achieves the best performance across almost all of the tested LLMs and tasks.} As shown in Figure \ref{new_exp2_2}, \dynamic{} outperforms DR.GRPO across all models and tasks, achieving a superior and more comprehensive performance. Although its math performance on Qwen2.5-Math-7B lags slightly behind \inter{}, it still makes a better performance than DR.GRPO. The radar chart illustrates the trade-off between two types of tasks faced by \inter{} and \intra{} between two types of tasks, as well as the balanced but mediocre performance of DR.GRPO. \inter{} tends to favor mathematical tasks, and \intra{} demonstrates strong performance in complex logical reasoning tasks. This once again highlights the effectiveness of \dynamic{}.

\vspace{15pt}
\subsection{Weighted Conditions for Efficient Reasoning.} \label{sec:exp3}

\textbf{Training Setup.} In this subsection, we utilize \methodName{} based on response length with $\mu=0.5$ in the Eq. \ref{eq:canon} and tune the $\alpha$ in the Eq. \ref{eq:weight}, where $C_q^+$ is considered the group with longer responses. A larger $\alpha$ means less compression of length. We follow the training setups described in Section \ref{sec:exp1} and reduce the maximum response length to 3072 for better efficiency. To be specific, we use \eff{} to denote the results of \methodName{} with weighted conditions of length.
%

\textbf{Evaluation Setup.} To systematically assess LRMs’ reasoning efficiency~\citep{qu2025surveyefficientreasoninglarge,lab2025safework}, we introduce two types of curves: \textbf{budget-performance curves for each LRM} and \textbf{cost-performance curves of different coefficients for all compared baselines}. Specifically, we set a maximum budget for each benchmark based on its difficulty and the average unconstrained output length of LRMs (Appendix \ref{app:max_budget}), then slice the same response at various budget ratios to draw the budget-performance curves. Moreover, we tune the length-controlling coefficients of each baseline to draw the cost-performance curves, recording their average performance and token cost to enable a comprehensive and fair comparison. In every comparison, the closer to the upper-left corner, the better (which represents high accuracy and high efficiency at the same time).
\begin{table*}[t]
\centering
\caption{The comparison between different methods towards efficient reasoning. Bold and \underline{underline} indicate the best and second-best results, respectively.}
\setlength{\tabcolsep}{2.5pt}
\renewcommand{\arraystretch}{1.3}
\resizebox{\textwidth}{!}{%
\begin{tabular}{lcccccccccccc|>{\columncolor{math!20}}c>{\columncolor{high!20}}c}
\toprule
 & \multicolumn{2}{c}{AIME 24} & \multicolumn{2}{c}{AIME 25} & \multicolumn{2}{c}{Olympiad} & \multicolumn{2}{c}{AMC} & \multicolumn{2}{c}{MATH-500} & \multicolumn{2}{c}{GSM8k} & \multicolumn{2}{c}{Overall} \\
 \cmidrule(lr){2-3} \cmidrule(lr){4-5} \cmidrule(lr){6-7} \cmidrule(lr){8-9} \cmidrule(lr){10-11}\cmidrule(lr){12-13} \cmidrule(lr){14-15} 
 & Acc & Tokens & Acc & Tokens & Acc & Tokens & Acc & Tokens & Acc & Tokens & Acc & Tokens & Acc & Tokens \\
\midrule
DR.GRPO &	29.0& 	1640 	&19.0 &	1586  &	49.0 &	1172 &	64.6 &	1214	&	85.8 	&728 & 91.9 	&349 & 56.6 & 1115
\\

\midrule
Clip Length  	 &	28.0 &	1177& 	\underline{18.3} &	1177  &	\underline{47.3} &	915  &\textbf{63.1} &	956  &	\underline{84.8} 	&612& \textbf{92.9}& 	291	&55.7 &	855
 \\

Length Reward$_{+}$ &	\textbf{31.7 }	&1190& 	18.0 &	1208 &	46.7 	&864  & 	61.8 &	937  &	84.6 	&546 & 91.9 &	255  &	\underline{56.2} &	869
 \\
Length Reward$_{*}$ &27.3 &	\underline{1087}	&13.7 &	\underline{1027}   &	46.4 &	\underline{707}  &	61.0 &	\underline{779} &	83.0 &	\underline{463} & \underline{92.2} &	\underline{198}	&	53.9 & \underline{710}
\\
\midrule
\eff{} ($\alpha=0.88$)		  &	27.3 	&\textbf{816} &	15.3 &	\textbf{862} &	43.9 &	\textbf{582} &59.3 &	\textbf{649} &84.4 	&\textbf{386}  &91.4 &	\textbf{166}  &	53.6 & \textbf{577}
   \\
\eff{} ($\alpha=0.96$) 	&\underline{29.7} &	1216 	&\textbf{19.0} &	1136  &	\textbf{48.4} &	881 &\underline{62.3}& 	936	&	\textbf{85.8} &	533  & 92.0 &	233 &	\textbf{56.2} & 822
 \\
\bottomrule
\end{tabular}%
}
\label{tab:main_transposed}
\vspace{-10pt}
\end{table*}

\begin{figure}[t]
    \centering
    \begin{subfigure}[b]{0.33\linewidth}
        \centering
        \includegraphics[width=\linewidth]{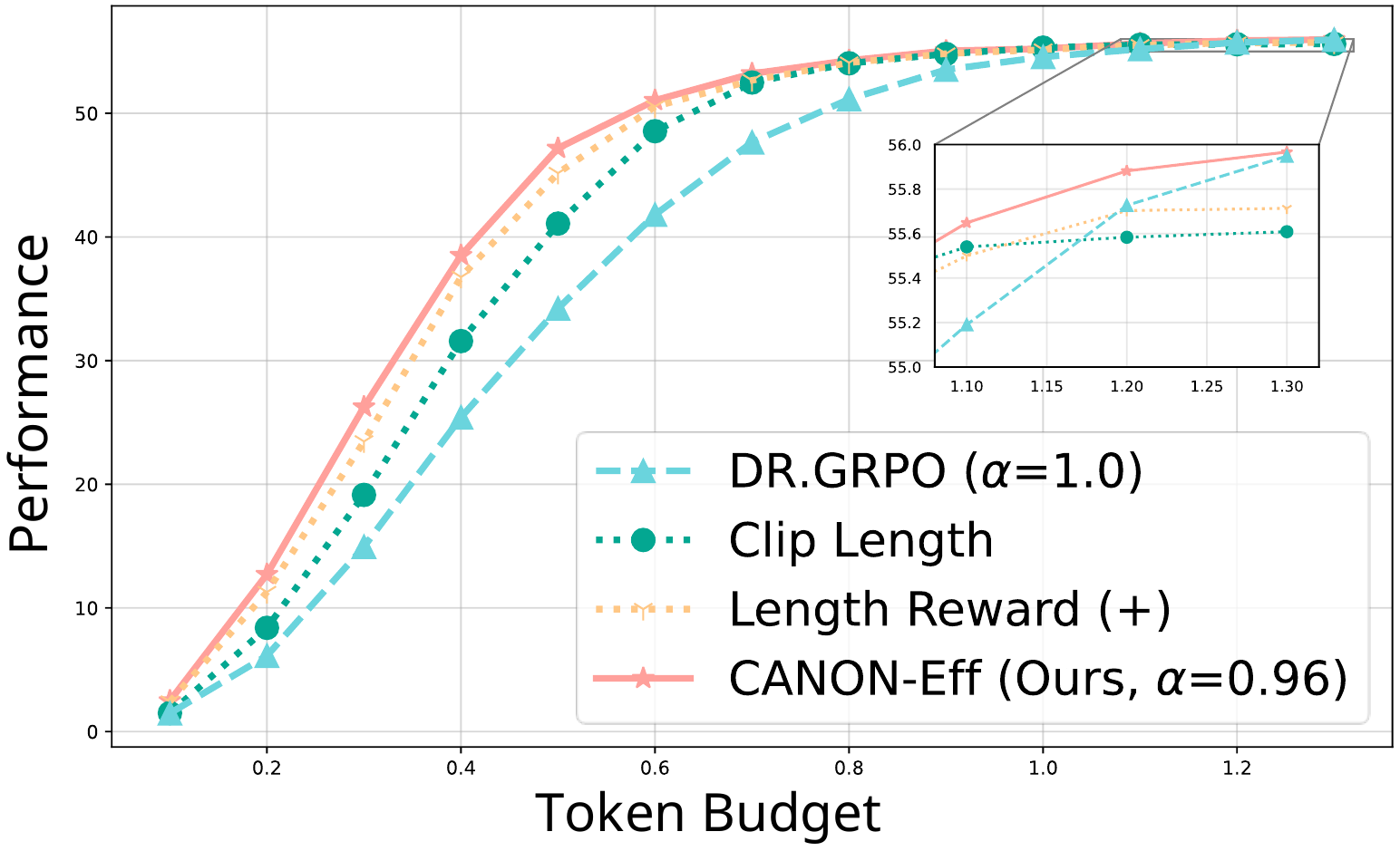}
        \caption{\eff{} with $\alpha=0.96$ consistently outperforms baselines methods.}
        \label{fig:left}
    \end{subfigure}
    \hfill
    \begin{subfigure}[b]{0.33\linewidth}
        \centering
        \includegraphics[width=\linewidth]{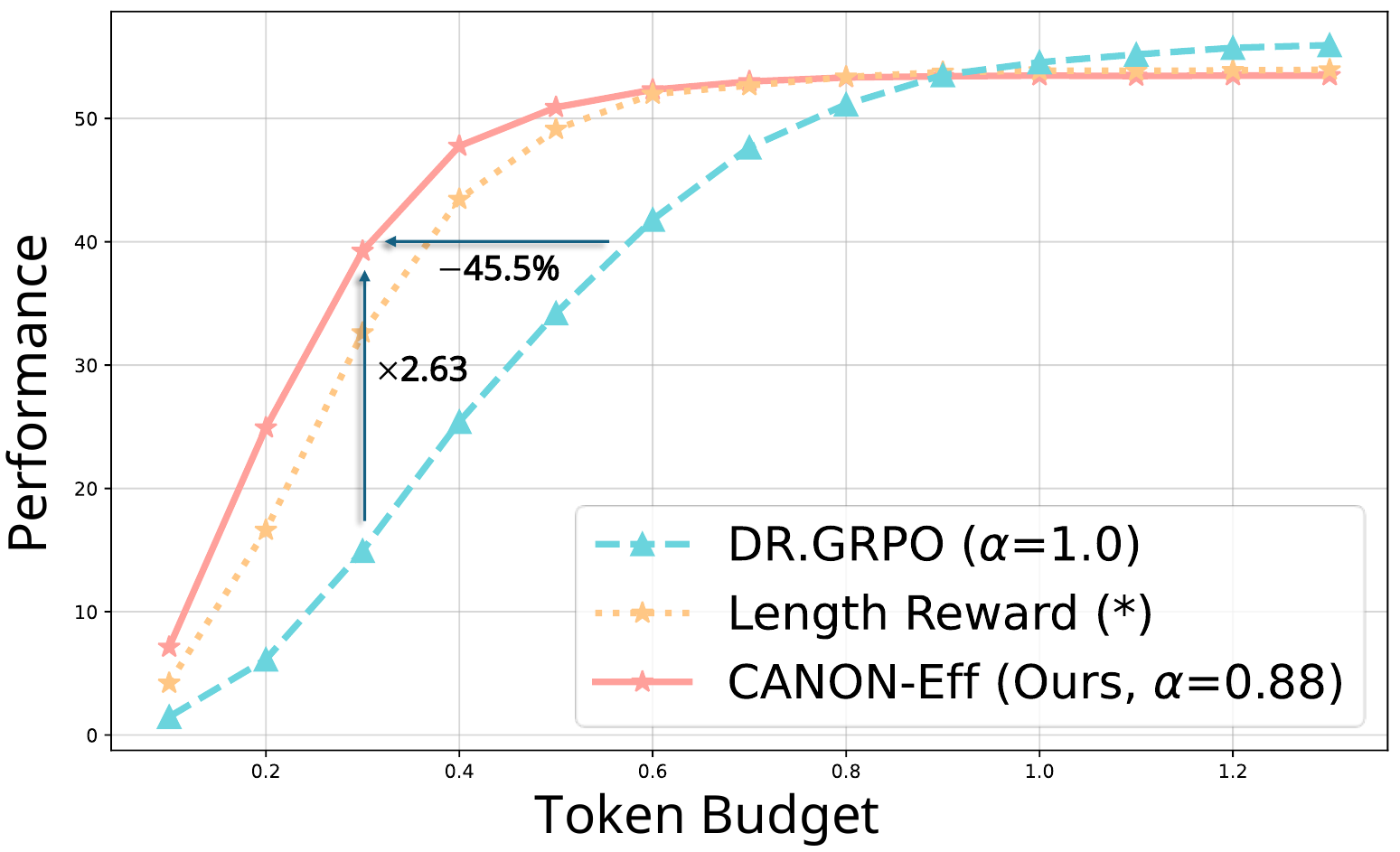}
        \caption{\eff{} with $\alpha=0.88$ achieves significantly better performance at low token budgets.}
        \label{fig:right}
    \end{subfigure}
    \hfill
    \begin{subfigure}[b]{0.3\linewidth}
        \centering
        \includegraphics[width=0.9\linewidth]{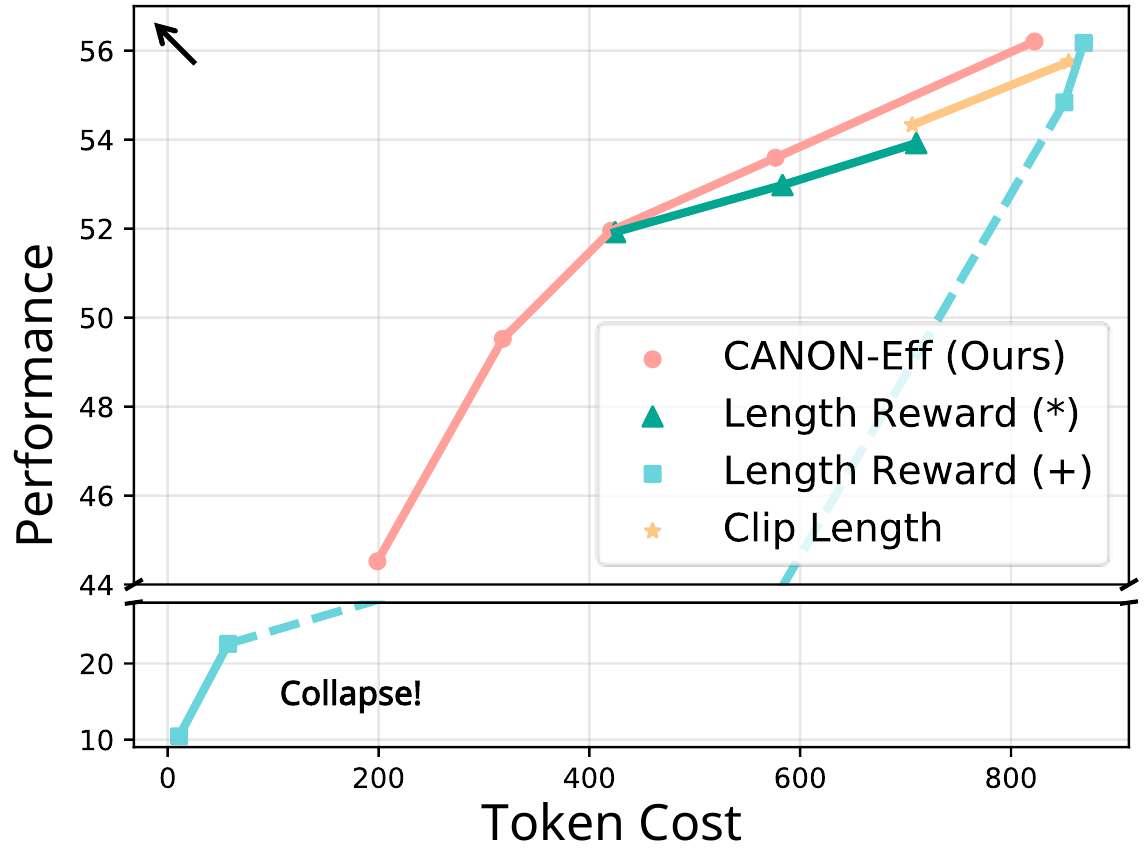}
        \caption{The Pareto frontier in the trade-off between performance and token efficiency.}
		\label{new_exp3}
    \end{subfigure}

    \caption{Budget-Performance and Cost-Performance Curves for Efficient Reasoning. This figure compares the reasoning efficiency of \eff{} against baselines under various token budgets. }
    \label{fig:combo}
    \vspace{-20pt}
\end{figure}

\textbf{Baselines.} We select three types of baseline methods towards efficient reasoning: (1) {Clip Length} that directly clips the maximum output length \citep{tp}, (2) {Length Reward ($+$)} that adds length penalties terms in the training reward (\cite{o1}, $+\text{coeff}*(\frac{\text{mean}_{G_q}(L)}{L}-1)$), and (3) {Length Reward ($*$)} that multiplies a normalized length coefficient on the reward (\cite{cmu1}, $* (1-\text{coeff} *\text{sigmoid}(\frac{L-\text{mean}_{G_q}(L)}{\text{std}_{G_q}(L)}))$). All these baselines are conducted with DR.GRPO. 

\textbf{\methodName{} achieves better performance with shorter responses compared with baselines.} We present the detailed performance of the top-performing models for each method across various benchmarks in Table \ref{tab:main_transposed}. \eff{} with $\alpha=0.96$ Pareto dominates the results of Clip Length and Length Reward ($+$), reducing the length by 26.3\% compared to DR.GRPO while only decreasing performance by 0.4 points. Figure \ref{fig:combo} shows that \eff{} with $\alpha=0.96$ consistently outperforms the baseline methods in both low-token-budget and high-token-budget scenarios. Since models trained with the {Length Reward ($*$)} exhibit significantly lower length with low performance at the same time, it is difficult to fairly compare with other baselines. To this end, we include an additional model trained with \eff{} with $\alpha=0.88$ that has comparable performance. \ref{fig:right} indicates that \methodName{} with $\alpha=0.88$ shows better token efficiency compared with {Length Reward ($*$)}, achieving 2.63 times the performance of DR.GRPO in low-token-budget scenarios, while reducing token consumption by 45.5\% at the same performance level.

\textbf{\methodName{} achieves a better Pareto frontier and stably explores the entire frontier.}
To draw the cost-performance curves for each method, we draw the Pareto frontier of \eff{} with the results of $\alpha=0.5,0.7,0.8,0.88, 0.96$. For Length Clipping, we respectively present the results with maximum lengths of 2048 and 1024 in the Pareto frontier. For Length Reward ($+$), penalty coefficients of 0.001, 0.004, 0.005, and 0.1 are used, respectively. For Length Reward ($*$), we utilize the coefficients of 0.05, 0.2, and 0.4. \ref{new_exp3} shows that all the frontier from baselines are dominated by the frontier of \eff{}'s. It is noteworthy that after the coefficient of Length Reward ($+$) is adjusted from 0.004 to 0.005, its performance drops from 54.8 to 22.5. In contrast, \eff{} remains consistently stable, exploring the Pareto frontier efficiently.




\vspace{-5pt}

\section{Analysis}
\vspace{-5pt}
\begin{wrapfigure}{r}{0.4\textwidth}
\vspace{-5pt} 
  \centering
  \captionof{table}{\small The accuracy and token cost of \inter{} with different metrics.}
  \vspace{-10pt} 
    \setlength{\tabcolsep}{2.5pt}  
    \setlength{\arrayrulewidth}{0.2pt}
    \renewcommand{\arraystretch}{1.2} 
    {\fontsize{2pt}{2.2pt}\selectfont
  \resizebox{0.38\textwidth}{!}{%
  \begin{tabular}{lcc}
    \hline
    Methods & \textit{Acc} & \textit{Tokens} \\
    \hline
    DR.GRPO &55.7 &	 1522 \\
    \hline
    Random  regrouping & 55.7 &	 1557 \\
    \rowcolor{gray!30}
    \multicolumn{3}{c}{\inter{}} \\
    {based on \textit{Length}}  & 55.3 &	\textbf{1008} \\
    {based on \textit{Entropy}}  & \textbf{57.6} & 1466 \\
    \hline
\end{tabular}
}}
\label{ablation}
\vspace{-10pt} 
\end{wrapfigure} 
In this section, we analyze how \dynamic{} and \eff{} effectively improve the task performance and reasoning efficiency.

\begin{wrapfigure}{r}{0.4\textwidth}
  \centering
  \vspace{-5pt} 
  \includegraphics[width=0.38\textwidth]{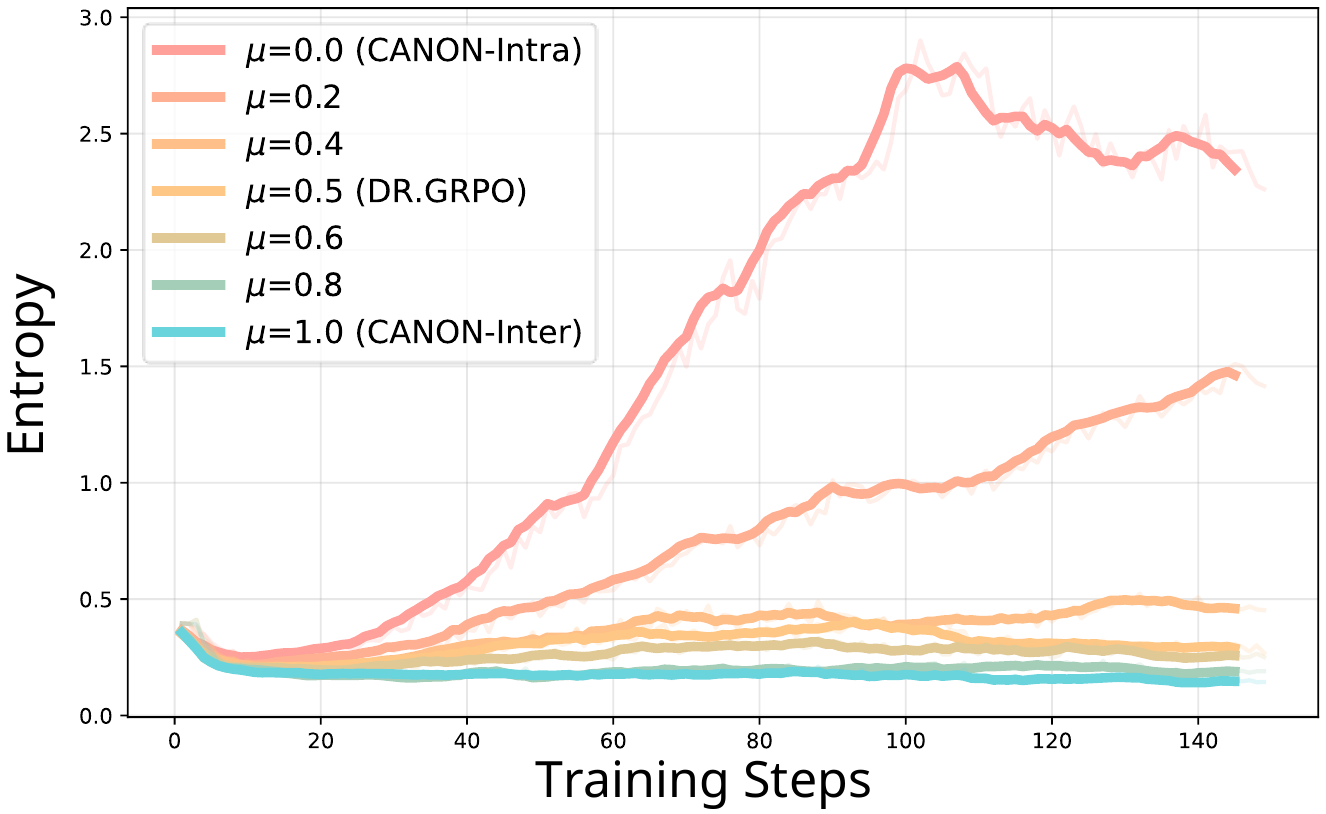} 
  \vspace{-10pt} 
  \captionof{figure}{\small\methodName{} shows hierarchical trends of target metrics through different combinations of \inter{} and \intra{}.} 
  \label{fig:ana1}
  \vspace{-10pt}
\end{wrapfigure}

\textbf{\methodName{} selects appropriate metrics as the target.} We conduct a simple ablation study on the target metrics considered by \methodName{}. As shown in Table \ref{ablation}, random regrouping achieves only the same performance as the baseline method while producing longer responses, thus failing to improve either performance or efficiency compared to the baseline. In contrast, \inter{} based on the response length excels in the token efficiency with 33.8\% shorter responses, and the entropy-based \inter{} delivers the best performance (57.6 points) among the comparisons.

\textbf{Different advantage combinations of \methodName{} select different trends of the target metrics.} Due to the different baseline rewards being compared, \inter{} tends to favor correct answers from the group with a higher average reward, while \intra{} selects correct answers from the group with a lower average reward. We compare the effects of \methodName{} on their target metrics across seven different settings, with $\mu$ ranging from 0.0 to 1.0. When entropy is considered, figure \ref{fig:ana1} shows that a larger $\mu$ (favoring more \inter{}) leads to a reduction in entropy, whereas a smaller $\mu$ (favoring more \intra{}) promotes an increase in entropy. The results demonstrates a hierarchical trend in the metric changes, indicating the effectiveness of controlling and selecting different trends from \inter{} and \intra{}. In this way, \dynamic{} can boost the task performance by adjusting different combinations of the two components.
\begin{wrapfigure}{r}{0.4\textwidth}
  \centering
  \vspace{-10pt} 
  \includegraphics[width=0.38\textwidth]{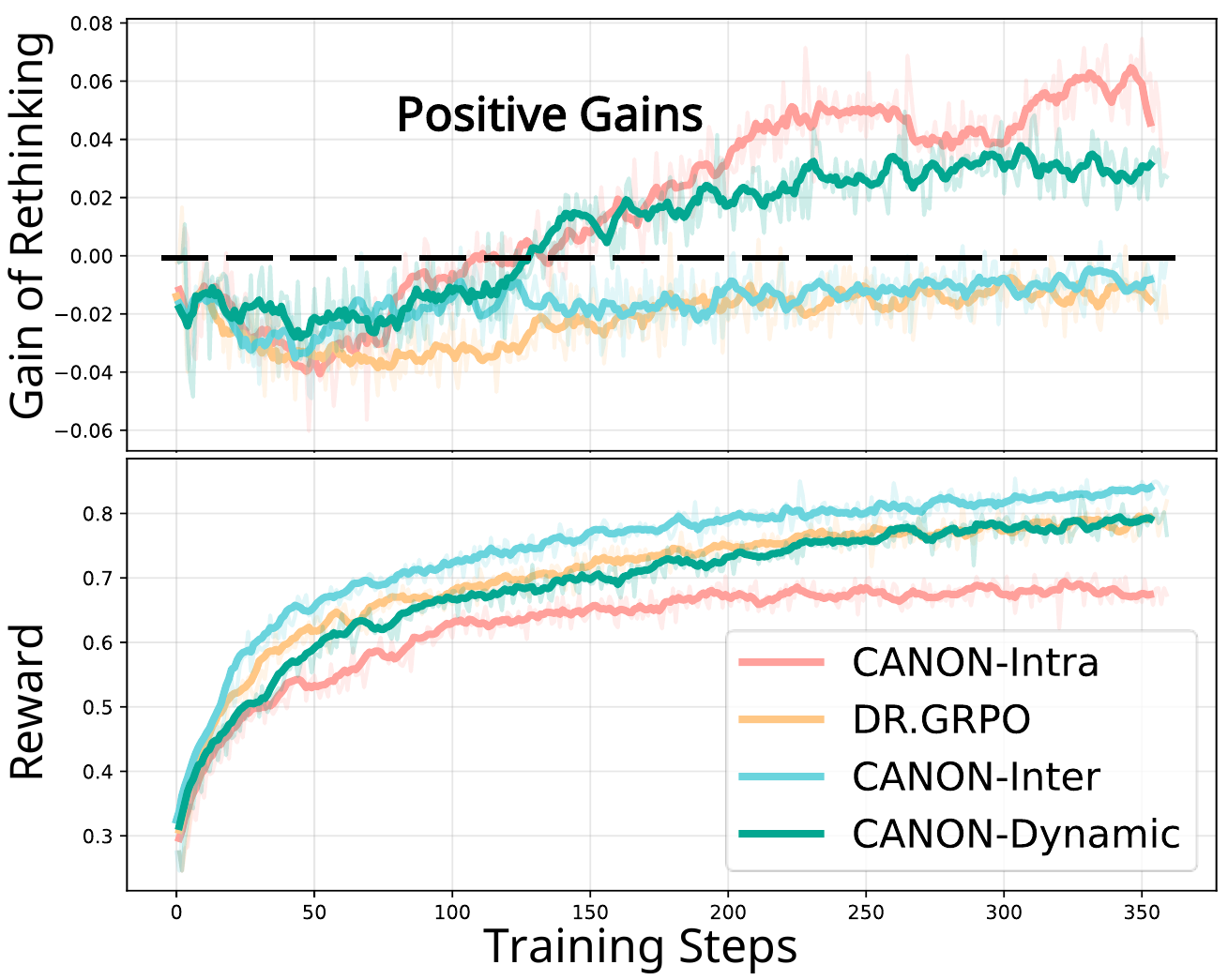} 
  \vspace{-10pt}
  \captionof{figure}{\small\dynamic{} with scheduled $\mu$ has positive gains of rethinking and high training score at the same time.} 
  \label{fig:ana3}
  \vspace{-15pt}
\end{wrapfigure}

\textbf{\methodName{} can achieve positive gains of more rethinking and high training efficiency through scheduling of two advantages.} As shown in Figure \ref{fig:ana3}, we record the performance genuinely brought by reflections and the curve of training reward. Although \intra{} achieves positive gains from more reflections, its training reward experiences a significant decline. In contrast, \inter{}, which shows a similar trend of DR.GRPO, has not yet achieved positive returns even by step 360, but maintains a higher training reward. \dynamic{}, on the other hand, not only achieves positive gains of rethinking but also makes a training reward on a par with \inter{}'s. This explains why \dynamic{} can achieve comprehensive leading performance in both math and complex logic reasoning tasks.

\vspace{5pt}
\section{Conclusion}

In this paper, we introduce \methodName{}, a novel reinforcement learning framework for large reasoning models that leverages human priors on training metrics (e.g., entropy, response length) without presuming their directional impact on performance. Extensive experiments across six math reasoning benchmarks and three high-complexity logic reasoning tasks demonstrate that CANON significantly outperforms prior advantage estimation methods like DR.GRPO. \methodName{} also supports flexible weighting of different metric trends, where \methodName{} based on
response length achieves a superior Pareto frontier in the performance-efficiency trade-off. Our analysis further confirms that \methodName{} promotes beneficial behaviors such as effective exploration and reflection, which are critical for solving complex reasoning problems.

\newpage
\section*{Ethics statement}
This work aims to introduce human priors about key metrics into reinforcement learning by proposing a novel advantage estimation framework named \methodName{}, which amplifies the impact of target metrics without presuming preferences. The experiments in this paper are limited to reasoning tasks conducted on open-source models, datasets, and benchmarks, which will not raise ethical concerns. We hope to explore the potential of \methodName{} to enhance the security of large language models in the future, thereby promoting their reliable and trustworthy development.

\section*{Reproducibility statement}
We aim to include both the high-level and low-level details of our method in the setup paragraphs of Section \ref{sec:exp} and Appendix \ref{details} to reproduce our results. All experiments are conducted on open-source LLMs and benchmarks. We employ open-source datasets for the Qwen series LLMs, provide a detailed description of the prompts used for training and evaluation, and comprehensively present the construction process of the training dataset for the Llama series LLM. Our code implementation is based on VeRL \citep{sheng2024hybridflow}, which is applied with focused modifications in the advantage computation part, enhancing the reproducibility of our work. Please access our code base via the following link: \href{https://github.com/biuboomc/CANON}{\methodName{}}.

\bibliography{iclr2026_conference}

\begin{thebibliography}{40}
\providecommand{\natexlab}[1]{#1}
\providecommand{\url}[1]{\texttt{#1}}
\expandafter\ifx\csname urlstyle\endcsname\relax
  \providecommand{\doi}[1]{doi: #1}\else
  \providecommand{\doi}{doi: \begingroup \urlstyle{rm}\Url}\fi

\bibitem[Ahmadian et~al.(2024)Ahmadian, Cremer, Gall{\'e}, Fadaee, Kreutzer, Pietquin, {\"U}st{\"u}n, and Hooker]{ahmadian2024back}
Arash Ahmadian, Chris Cremer, Matthias Gall{\'e}, Marzieh Fadaee, Julia Kreutzer, Olivier Pietquin, Ahmet {\"U}st{\"u}n, and Sara Hooker.
\newblock Back to basics: Revisiting reinforce style optimization for learning from human feedback in llms.
\newblock \emph{arXiv preprint arXiv:2402.14740}, 2024.

\bibitem[Arora \& Zanette(2025)Arora and Zanette]{cmu1}
Daman Arora and Andrea Zanette.
\newblock Training language models to reason efficiently.
\newblock \emph{arXiv preprint arXiv:2502.04463}, 2025.

\bibitem[Chen et~al.(2025{\natexlab{a}})Chen, Li, Gong, Jiang, Fei, Yang, Shan, Yu, Wang, Zhu, et~al.]{chen2025minimax}
Aili Chen, Aonian Li, Bangwei Gong, Binyang Jiang, Bo~Fei, Bo~Yang, Boji Shan, Changqing Yu, Chao Wang, Cheng Zhu, et~al.
\newblock Minimax-m1: Scaling test-time compute efficiently with lightning attention.
\newblock \emph{arXiv preprint arXiv:2506.13585}, 2025{\natexlab{a}}.

\bibitem[Chen et~al.(2025{\natexlab{b}})Chen, Chen, Wang, and Yang]{chen2025seed}
Minghan Chen, Guikun Chen, Wenguan Wang, and Yi~Yang.
\newblock Seed-grpo: Semantic entropy enhanced grpo for uncertainty-aware policy optimization.
\newblock \emph{arXiv preprint arXiv:2505.12346}, 2025{\natexlab{b}}.

\bibitem[Cheng et~al.(2025)Cheng, Huang, Zhu, Dai, Zhao, Zhang, and Wei]{cheng2025reasoning}
Daixuan Cheng, Shaohan Huang, Xuekai Zhu, Bo~Dai, Wayne~Xin Zhao, Zhenliang Zhang, and Furu Wei.
\newblock Reasoning with exploration: An entropy perspective.
\newblock \emph{arXiv preprint arXiv:2506.14758}, 2025.

\bibitem[Cobbe et~al.(2021)Cobbe, Kosaraju, Bavarian, Chen, Jun, Kaiser, Plappert, Tworek, Hilton, Nakano, et~al.]{cobbe2021training}
Karl Cobbe, Vineet Kosaraju, Mohammad Bavarian, Mark Chen, Heewoo Jun, Lukasz Kaiser, Matthias Plappert, Jerry Tworek, Jacob Hilton, Reiichiro Nakano, et~al.
\newblock Training verifiers to solve math word problems.
\newblock \emph{arXiv preprint arXiv:2110.14168}, 2021.

\bibitem[Comanici et~al.(2025)Comanici, Bieber, Schaekermann, Pasupat, Sachdeva, Dhillon, Blistein, Ram, Zhang, Rosen, et~al.]{gemini}
Gheorghe Comanici, Eric Bieber, Mike Schaekermann, Ice Pasupat, Noveen Sachdeva, Inderjit Dhillon, Marcel Blistein, Ori Ram, Dan Zhang, Evan Rosen, et~al.
\newblock Gemini 2.5: Pushing the frontier with advanced reasoning, multimodality, long context, and next generation agentic capabilities.
\newblock \emph{arXiv preprint arXiv:2507.06261}, 2025.

\bibitem[Dubey et~al.(2024)Dubey, Jauhri, Pandey, Kadian, Al-Dahle, Letman, Mathur, Schelten, Yang, Fan, et~al.]{dubey2024llama}
Abhimanyu Dubey, Abhinav Jauhri, Abhinav Pandey, Abhishek Kadian, Ahmad Al-Dahle, Aiesha Letman, Akhil Mathur, Alan Schelten, Amy Yang, Angela Fan, et~al.
\newblock The llama 3 herd of models.
\newblock \emph{arXiv e-prints}, pp.\  arXiv--2407, 2024.

\bibitem[Gandhi et~al.(2025)Gandhi, Chakravarthy, Singh, Lile, and Goodman]{gandhi2025cognitive}
Kanishk Gandhi, Ayush Chakravarthy, Anikait Singh, Nathan Lile, and Noah~D Goodman.
\newblock Cognitive behaviors that enable self-improving reasoners, or, four habits of highly effective stars.
\newblock \emph{arXiv preprint arXiv:2503.01307}, 2025.

\bibitem[Guo et~al.(2025)Guo, Yang, Zhang, Song, Zhang, Xu, Zhu, Ma, Wang, Bi, et~al.]{guo2025deepseek}
Daya Guo, Dejian Yang, Haowei Zhang, Junxiao Song, Ruoyu Zhang, Runxin Xu, Qihao Zhu, Shirong Ma, Peiyi Wang, Xiao Bi, et~al.
\newblock Deepseek-r1: Incentivizing reasoning capability in llms via reinforcement learning.
\newblock \emph{arXiv preprint arXiv:2501.12948}, 2025.

\bibitem[Hassid et~al.(2025)Hassid, Synnaeve, Adi, and Schwartz]{hassid2025don}
Michael Hassid, Gabriel Synnaeve, Yossi Adi, and Roy Schwartz.
\newblock Don't overthink it. preferring shorter thinking chains for improved llm reasoning.
\newblock \emph{arXiv preprint arXiv:2505.17813}, 2025.

\bibitem[He et~al.(2024)He, Luo, Bai, Hu, Thai, Shen, Hu, Han, Huang, Zhang, et~al.]{dataset_olympiad}
Chaoqun He, Renjie Luo, Yuzhuo Bai, Shengding Hu, Zhen Thai, Junhao Shen, Jinyi Hu, Xu~Han, Yujie Huang, Yuxiang Zhang, et~al.
\newblock Olympiadbench: A challenging benchmark for promoting agi with olympiad-level bilingual multimodal scientific problems.
\newblock In \emph{Proceedings of the 62nd Annual Meeting of the Association for Computational Linguistics (Volume 1: Long Papers)}, pp.\  3828--3850, 2024.

\bibitem[He et~al.(2025)He, Liang, Xu, Liu, Chen, Wang, Song, Yu, Liang, Wang, et~al.]{he2025deepmath}
Zhiwei He, Tian Liang, Jiahao Xu, Qiuzhi Liu, Xingyu Chen, Yue Wang, Linfeng Song, Dian Yu, Zhenwen Liang, Wenxuan Wang, et~al.
\newblock Deepmath-103k: A large-scale, challenging, decontaminated, and verifiable mathematical dataset for advancing reasoning.
\newblock \emph{arXiv preprint arXiv:2504.11456}, 2025.

\bibitem[Hendrycks et~al.(2021)Hendrycks, Burns, Kadavath, Arora, Basart, Tang, Song, and Steinhardt]{dataset_math}
Dan Hendrycks, Collin Burns, Saurav Kadavath, Akul Arora, Steven Basart, Eric Tang, Dawn Song, and Jacob Steinhardt.
\newblock Measuring mathematical problem solving with the math dataset.
\newblock \emph{arXiv preprint arXiv:2103.03874}, 2021.

\bibitem[Hou et~al.(2025)Hou, Zhang, Ji, Liu, Qian, Andreas, and Chang]{tp}
Bairu Hou, Yang Zhang, Jiabao Ji, Yujian Liu, Kaizhi Qian, Jacob Andreas, and Shiyu Chang.
\newblock Thinkprune: Pruning long chain-of-thought of llms via reinforcement learning.
\newblock \emph{arXiv preprint arXiv:2504.01296}, 2025.

\bibitem[Hu(2025)]{hu2025reinforce++}
Jian Hu.
\newblock Reinforce++: A simple and efficient approach for aligning large language models.
\newblock \emph{arXiv preprint arXiv:2501.03262}, 2025.

\bibitem[Hu et~al.(2025)Hu, Zhang, Han, Jiang, Zhang, and Shum]{orz}
Jingcheng Hu, Yinmin Zhang, Qi~Han, Daxin Jiang, Xiangyu Zhang, and Heung-Yeung Shum.
\newblock Open-reasoner-zero: An open source approach to scaling up reinforcement learning on the base model, 2025.
\newblock URL \url{https://arxiv.org/abs/2503.24290}.

\bibitem[{Hugging Face}(2025)]{openr1}
{Hugging Face}.
\newblock Open r1: A fully open reproduction of deepseek-r1, January 2025.
\newblock URL \url{https://github.com/huggingface/open-r1}.

\bibitem[Jaech et~al.(2024)Jaech, Kalai, Lerer, Richardson, El-Kishky, Low, Helyar, Madry, Beutel, Carney, et~al.]{o1_card}
Aaron Jaech, Adam Kalai, Adam Lerer, Adam Richardson, Ahmed El-Kishky, Aiden Low, Alec Helyar, Aleksander Madry, Alex Beutel, Alex Carney, et~al.
\newblock Openai o1 system card.
\newblock \emph{arXiv preprint arXiv:2412.16720}, 2024.

\bibitem[Lab et~al.(2025)Lab, Bao, Chen, Chen, Chen, Chen, Chen, Chen, Chen, Cheng, et~al.]{lab2025safework}
Shanghai~AI Lab, Yicheng Bao, Guanxu Chen, Mingkang Chen, Yunhao Chen, Chiyu Chen, Lingjie Chen, Sirui Chen, Xinquan Chen, Jie Cheng, et~al.
\newblock Safework-r1: Coevolving safety and intelligence under the ai-45$^\circ$ law.
\newblock \emph{arXiv preprint arXiv:2507.18576}, 2025.

\bibitem[Lambert et~al.(2024)Lambert, Morrison, Pyatkin, Huang, Ivison, Brahman, Miranda, Liu, Dziri, Lyu, et~al.]{lambert2024tulu}
Nathan Lambert, Jacob Morrison, Valentina Pyatkin, Shengyi Huang, Hamish Ivison, Faeze Brahman, Lester James~V Miranda, Alisa Liu, Nouha Dziri, Shane Lyu, et~al.
\newblock Tulu 3: Pushing frontiers in open language model post-training.
\newblock \emph{arXiv preprint arXiv:2411.15124}, 2024.

\bibitem[Li et~al.(2024)Li, Beeching, Tunstall, Lipkin, Soletskyi, Huang, Rasul, Yu, Jiang, Shen, et~al.]{li2024numinamath}
Jia Li, Edward Beeching, Lewis Tunstall, Ben Lipkin, Roman Soletskyi, Shengyi Huang, Kashif Rasul, Longhui Yu, Albert~Q. Jiang, Ziju Shen, et~al.
\newblock Numinamath: The largest public dataset in ai4maths with 860k pairs of competition math problems and solutions.
\newblock \url{https://huggingface.co/datasets/Numinamath}, 2024.
\newblock Hugging Face repository, 13:9.

\bibitem[Li et~al.(2023)Li, Xu, Zhang, Yu, Sun, and Luo]{li2023remax}
Ziniu Li, Tian Xu, Yushun Zhang, Yang Yu, Ruoyu Sun, and Zhi-Quan Luo.
\newblock Remax: A simple, effective, and efficient method for aligning large language models.
\newblock \emph{arXiv preprint arXiv:2310.10505}, 2023.

\bibitem[Lin et~al.(2025)Lin, Bras, Richardson, Sabharwal, Poovendran, Clark, and Choi]{zebra}
Bill~Yuchen Lin, Ronan~Le Bras, Kyle Richardson, Ashish Sabharwal, Radha Poovendran, Peter Clark, and Yejin Choi.
\newblock Zebralogic: On the scaling limits of llms for logical reasoning.
\newblock \emph{arXiv preprint arXiv:2502.01100}, 2025.

\bibitem[Liu et~al.(2025{\natexlab{a}})Liu, Chen, Li, Qi, Pang, Du, Lee, and Lin]{liu2025understanding}
Zichen Liu, Changyu Chen, Wenjun Li, Penghui Qi, Tianyu Pang, Chao Du, Wee~Sun Lee, and Min Lin.
\newblock Understanding r1-zero-like training: A critical perspective.
\newblock \emph{arXiv preprint arXiv:2503.20783}, 2025{\natexlab{a}}.

\bibitem[Liu et~al.(2025{\natexlab{b}})Liu, Meng, Du, Zhou, Yu, Shao, and Zhang]{liu2025cpgd}
Zongkai Liu, Fanqing Meng, Lingxiao Du, Zhixiang Zhou, Chao Yu, Wenqi Shao, and Qiaosheng Zhang.
\newblock Cpgd: Toward stable rule-based reinforcement learning for language models.
\newblock \emph{arXiv preprint arXiv:2505.12504}, 2025{\natexlab{b}}.

\bibitem[Luo et~al.(2025)Luo, Shen, He, Wang, Liu, Li, Tan, Cao, and Tao]{o1}
Haotian Luo, Li~Shen, Haiying He, Yibo Wang, Shiwei Liu, Wei Li, Naiqiang Tan, Xiaochun Cao, and Dacheng Tao.
\newblock O1-pruner: Length-harmonizing fine-tuning for o1-like reasoning pruning.
\newblock \emph{arXiv preprint arXiv:2501.12570}, 2025.

\bibitem[Prabhudesai et~al.(2025)Prabhudesai, Chen, Ippoliti, Fragkiadaki, Liu, and Pathak]{prabhudesai2025maximizing}
Mihir Prabhudesai, Lili Chen, Alex Ippoliti, Katerina Fragkiadaki, Hao Liu, and Deepak Pathak.
\newblock Maximizing confidence alone improves reasoning.
\newblock \emph{arXiv preprint arXiv:2505.22660}, 2025.

\bibitem[Qu et~al.(2025)Qu, Li, Su, Sun, Yan, Liu, Cui, Liu, Liang, He, Li, Wei, Shao, Lu, Zhang, Hua, Zhou, and Cheng]{qu2025surveyefficientreasoninglarge}
Xiaoye Qu, Yafu Li, Zhaochen Su, Weigao Sun, Jianhao Yan, Dongrui Liu, Ganqu Cui, Daizong Liu, Shuxian Liang, Junxian He, Peng Li, Wei Wei, Jing Shao, Chaochao Lu, Yue Zhang, Xian-Sheng Hua, Bowen Zhou, and Yu~Cheng.
\newblock A survey of efficient reasoning for large reasoning models: Language, multimodality, and beyond, 2025.
\newblock URL \url{https://arxiv.org/abs/2503.21614}.

\bibitem[Schulman et~al.(2015)Schulman, Moritz, Levine, Jordan, and Abbeel]{schulman2015high}
John Schulman, Philipp Moritz, Sergey Levine, Michael Jordan, and Pieter Abbeel.
\newblock High-dimensional continuous control using generalized advantage estimation.
\newblock \emph{arXiv preprint arXiv:1506.02438}, 2015.

\bibitem[Schulman et~al.(2017)Schulman, Wolski, Dhariwal, Radford, and Klimov]{ppo}
John Schulman, Filip Wolski, Prafulla Dhariwal, Alec Radford, and Oleg Klimov.
\newblock Proximal policy optimization algorithms.
\newblock \emph{arXiv preprint arXiv:1707.06347}, 2017.

\bibitem[Shao et~al.(2024)Shao, Wang, Zhu, Xu, Song, Bi, Zhang, Zhang, Li, Wu, et~al.]{shao2024deepseekmath}
Zhihong Shao, Peiyi Wang, Qihao Zhu, Runxin Xu, Junxiao Song, Xiao Bi, Haowei Zhang, Mingchuan Zhang, YK~Li, Y~Wu, et~al.
\newblock Deepseekmath: Pushing the limits of mathematical reasoning in open language models.
\newblock \emph{arXiv preprint arXiv:2402.03300}, 2024.

\bibitem[Sheng et~al.(2024)Sheng, Zhang, Ye, Wu, Zhang, Zhang, Peng, Lin, and Wu]{sheng2024hybridflow}
Guangming Sheng, Chi Zhang, Zilingfeng Ye, Xibin Wu, Wang Zhang, Ru~Zhang, Yanghua Peng, Haibin Lin, and Chuan Wu.
\newblock Hybridflow: A flexible and efficient rlhf framework.
\newblock \emph{arXiv preprint arXiv: 2409.19256}, 2024.

\bibitem[Wang et~al.(2025)Wang, Yu, Gao, Zheng, Liu, Lu, Dang, Chen, Yang, Zhang, et~al.]{wang2025beyond}
Shenzhi Wang, Le~Yu, Chang Gao, Chujie Zheng, Shixuan Liu, Rui Lu, Kai Dang, Xionghui Chen, Jianxin Yang, Zhenru Zhang, et~al.
\newblock Beyond the 80/20 rule: High-entropy minority tokens drive effective reinforcement learning for llm reasoning.
\newblock \emph{arXiv preprint arXiv:2506.01939}, 2025.

\bibitem[Yan et~al.(2025)Yan, Li, Hu, Wang, Cui, Qu, Cheng, and Zhang]{yan2025learning}
Jianhao Yan, Yafu Li, Zican Hu, Zhi Wang, Ganqu Cui, Xiaoye Qu, Yu~Cheng, and Yue Zhang.
\newblock Learning to reason under off-policy guidance.
\newblock \emph{arXiv preprint arXiv:2504.14945}, 2025.

\bibitem[Yang et~al.(2024)Yang, Zhang, Hui, Gao, Yu, Li, Liu, Tu, Zhou, Lin, Lu, Xue, Lin, Liu, Ren, and Zhang]{qwen2.5_math}
An~Yang, Beichen Zhang, Binyuan Hui, Bofei Gao, Bowen Yu, Chengpeng Li, Dayiheng Liu, Jianhong Tu, Jingren Zhou, Junyang Lin, Keming Lu, Mingfeng Xue, Runji Lin, Tianyu Liu, Xingzhang Ren, and Zhenru Zhang.
\newblock Qwen2.5-math technical report: Toward mathematical expert model via self-improvement, 2024.
\newblock URL \url{https://arxiv.org/abs/2409.12122}.

\bibitem[Yang et~al.(2025)Yang, Li, Yang, Zhang, Hui, Zheng, Yu, Gao, Huang, Lv, et~al.]{yang2025qwen3}
An~Yang, Anfeng Li, Baosong Yang, Beichen Zhang, Binyuan Hui, Bo~Zheng, Bowen Yu, Chang Gao, Chengen Huang, Chenxu Lv, et~al.
\newblock Qwen3 technical report.
\newblock \emph{arXiv preprint arXiv:2505.09388}, 2025.

\bibitem[Yu et~al.(2025)Yu, Zhang, Zhu, Yuan, Zuo, Yue, Dai, Fan, Liu, Liu, et~al.]{yu2025dapo}
Qiying Yu, Zheng Zhang, Ruofei Zhu, Yufeng Yuan, Xiaochen Zuo, Yu~Yue, Weinan Dai, Tiantian Fan, Gaohong Liu, Lingjun Liu, et~al.
\newblock Dapo: An open-source llm reinforcement learning system at scale.
\newblock \emph{arXiv preprint arXiv:2503.14476}, 2025.

\bibitem[Zeng et~al.(2025)Zeng, Huang, Liu, Liu, He, Ma, and He]{zeng2025simplerl}
Weihao Zeng, Yuzhen Huang, Qian Liu, Wei Liu, Keqing He, Zejun Ma, and Junxian He.
\newblock Simplerl-zoo: Investigating and taming zero reinforcement learning for open base models in the wild.
\newblock \emph{arXiv preprint arXiv:2503.18892}, 2025.

\bibitem[Zhang et~al.(2025)Zhang, Zuo, He, Sun, Liu, Jiang, Fan, Tian, Jia, Li, et~al.]{zhang2025survey}
Kaiyan Zhang, Yuxin Zuo, Bingxiang He, Youbang Sun, Runze Liu, Che Jiang, Yuchen Fan, Kai Tian, Guoli Jia, Pengfei Li, et~al.
\newblock A survey of reinforcement learning for large reasoning models.
\newblock \emph{arXiv preprint arXiv:2509.08827}, 2025.

\end{thebibliography}
\bibliographystyle{iclr2026_conference}

\appendix
\newpage
\section{Limitations.}
Based on feasibility and motivation, this work focuses on conditions that can be specified through numerical ordering, without exploring conditions that are more complex and harder to verify. Due to limitations in paper length and computation resources, this work primarily conducts the \methodName{} based on two metrics—response length and entropy—while other training metrics remain unexplored. Additionally, the paper considers only one metric at a time, without attempting to incorporate multiple metrics simultaneously. This demonstrates that the perspective and framework proposed in this work is flexible and hold significant potential for extension, which can be further explored in future research.

\section{The Use of Large Language Models.}
LLMs primarily assist this work in two aspects: on one hand, they are used for aiding our writing, and on the other hand, they sometimes serve as a coding assistant during the programming of our code base.

\section{Experiments Details.}
\label{details}
\subsection{Rethinking Patterns.} \label{app:patterns}
Following \cite{gandhi2025cognitive}, we firstly samples 10000 responses of Qwen3-32B \cite{yang2025qwen3} and utilize the modified prompts from \citep{gandhi2025cognitive} to collect the rethinking patterns of verification, sub-goal setting, and backtracking. Then we match these patterns in a few Question-Answer instances and filter out overly frequent conjunctions, overly short words, and semantically ambiguous phrases. The number of remaining keywords and regular expressions is 334 for verification,  1036 for sub-goal setting, and 532 for backtracking.

\subsection{The Maximum Token Budget Setups.} \label{app:max_budget}
We set the maximum token budget for each benchmark based on its difficulty and the average token length observed from models trained with DR.GRPO, as shown in Figure \ref{budget}. When plotting the performance-budget curve, we normalize the maximum token budget of each benchmark to 1.0. We then evaluate the performance of all benchmarks under token budgets ranging from 0.1× to 1.3× their respective maximum budget, averaging the results across benchmarks at each budget ratio and displaying them in the figure.
\begin{table}[h]
\centering
\caption{Benchmark-wise Maximum Token Budget.}
\begin{tabular}{lcc}
\toprule
\textbf{Benchmark} & \textbf{Avg.~Tokens (unlimited)} & \textbf{Max Token Budget} \\
\midrule
GSM8k            & 349  & 600  \\
MATH-500         & 728  & 1500 \\
AMC              & 1214 & 1800 \\
OlympiadBench    & 1172 & 1800 \\
AIME 2024        & 1640 & 2000 \\
AIME 2025        & 1586 & 2000 \\
\bottomrule
\end{tabular}
\label{budget}
\end{table}

\subsection{Reasons for expanding the context window of models from Qwen2.5-Math series.}
\label{app:sch3}
Initially, we uses the setting of Section \ref{sec:exp1}; however, during the training process, too much length clipping (> 30\%) results in nearly incomparable experimental outcomes, as shown in Figure \ref{fig:clip_length}. Therefore, we expand Qwen2.5-Math-7B's context limit from 4096 to 16384 and set the maximum output length to 8192, which alleviates this phenomenon.

\begin{figure}[htbp]
	\centering
	\begin{minipage}{0.33\linewidth}
		\centering
		\includegraphics[width=0.9\linewidth]{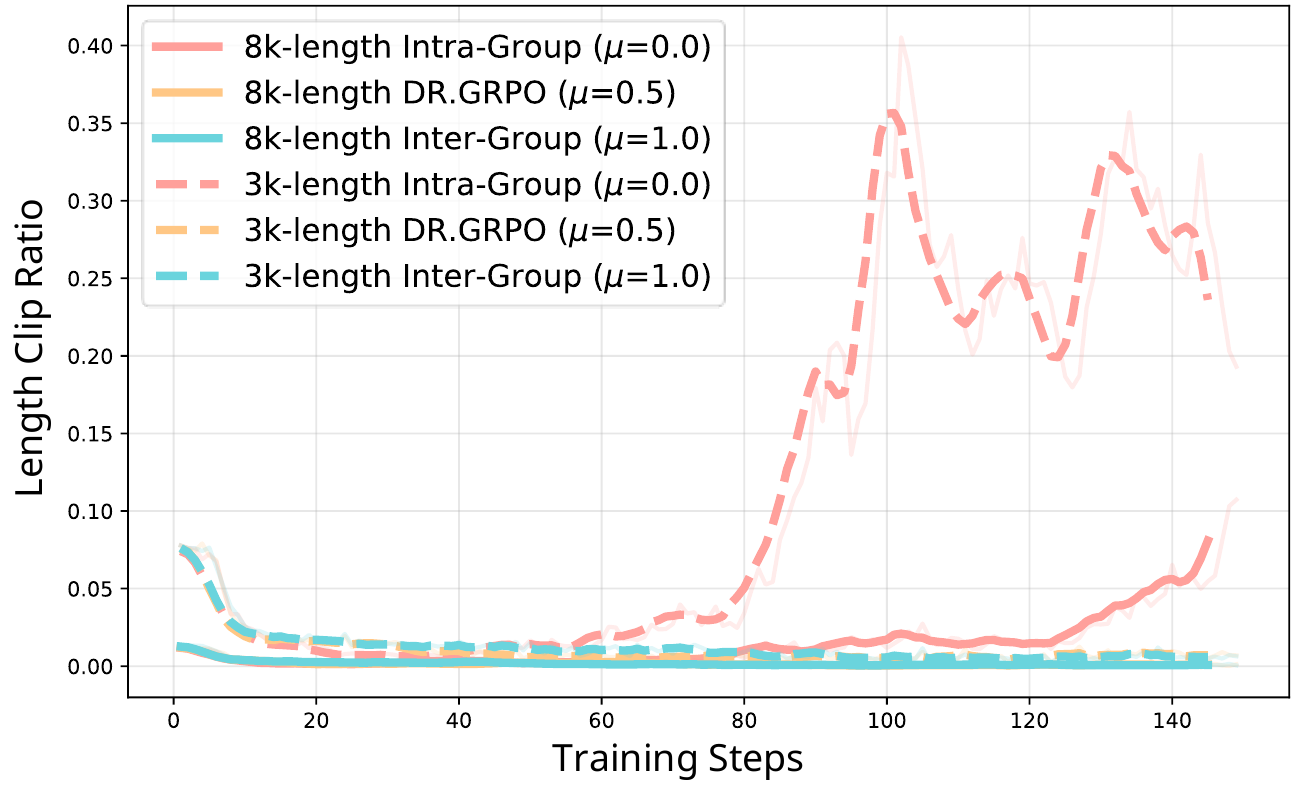}
		\caption{The ratio of answers truncated due to reaching the maximum output length.}
		\label{fig:clip_length}
	\end{minipage}
    \hfill
	\begin{minipage}{0.64\linewidth}
		\centering
		\includegraphics[width=0.9\linewidth]{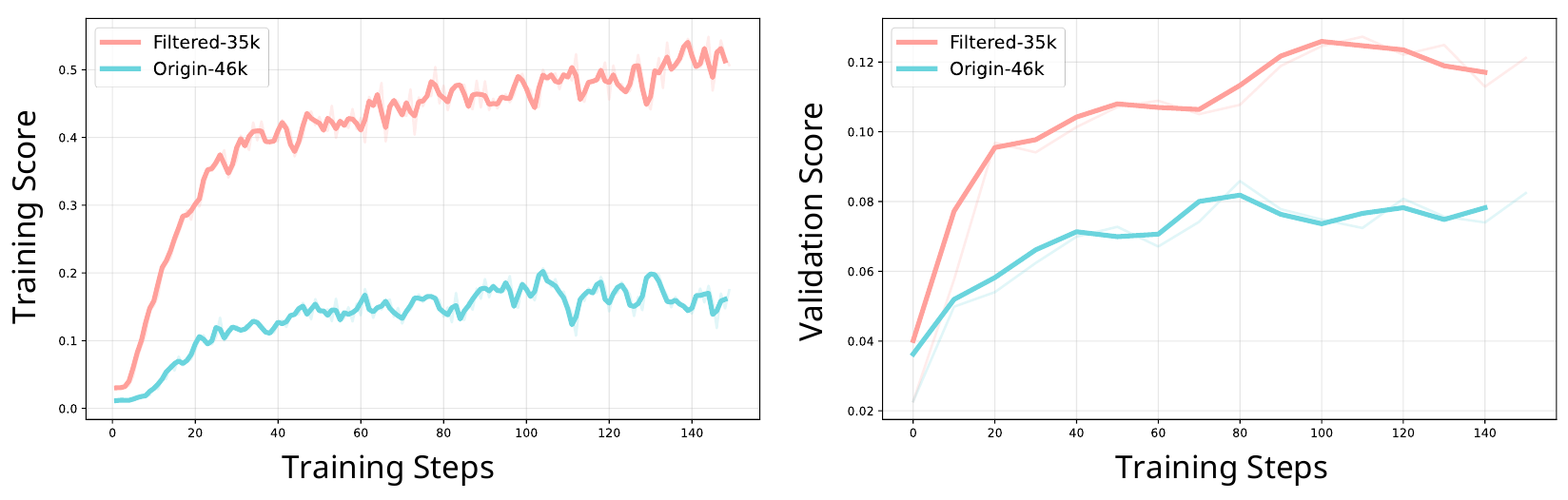}
		\caption{The score curves of the training set and validation set from the newly constructed dataset with 35k data and the original dataset used for the Qwen series models, respectively.}
		\label{fig:new_data}
	\end{minipage}
\end{figure}

\subsection{System prompt.}
\label{app:sch4}
For the training and inference of Qwen series models, we share the same system prompt as follows.
\begin{tcolorbox}[
    center,
    arc=0mm,
    boxrule=1pt,
    colback=blue!6!white,
    colframe=black,
    colbacktitle=black,
    attach boxed title to top left={yshift=-0.1in,xshift=0.15in},
    boxed title style={boxrule=0pt,colframe=white}
]
Your task is to follow a systematic, thorough reasoning process before providing the final solution. This involves analyzing, summarizing, exploring, reassessing, and refining your thought process through multiple iterations. Structure your response into two sections: Thought and Solution. In the Thought section, present your reasoning using the format: “\verb|<think>\n| {thoughts} \verb|</think>\n|”. Each thought should include detailed analysis, brainstorming, verification, and refinement of ideas. After “\verb|</think>\n|” in the Solution section, provide the final, logical, and accurate answer, clearly derived from the exploration in the Thought section. If applicable, include the answer in \verb|\boxed{}| for closed-form results like multiple choices or mathematical solutions. 
\end{tcolorbox}

\subsection{Construction of training dataset for Llama3.1-8B.}
\label{app:sch5}
Since the pretraining of Llama3.1-8B lacks data for long chain-of-thought and mathematical reasoning, its average training reward based on the original dataset used for Qwen2.5-Math remains below 0.2. To enhance training efficiency, we employ three Llama series models (Llama3.1-8B, Llama3.1-8B-Instruct, and Llama3.1-70B) to generate solutions for each problem across four datasets (training set of GSM8k \citep{cobbe2021training}, training set of MATH \citep{dataset_math}, a 46k subset of OpenR1-Math-220k \citep{orz,yan2025learning}, and DeepMath-103k \citep{he2025deepmath}). We then filter out questions whose accuracy of $\text{Pass@8} > 0$, ultimately selecting 35k samples for training the Llama3.1-8B model. Concurrently, due to Llama3.1-8B’s limited instruction-following capability, we simplify the output format requirements in its system prompt. 
\begin{tcolorbox}[
    center,
    arc=0mm,
    boxrule=1pt,
    colback=blue!6!white,
    colframe=black,
    colbacktitle=black,
    attach boxed title to top left={yshift=-0.1in,xshift=0.15in},
    boxed title style={boxrule=0pt,colframe=white}
]
Your task is to follow a systematic, thorough reasoning process before providing the final solution. This involves analyzing, summarizing, exploring, reassessing, and refining your thought process through multiple iterations. Structure your response into two sections: Thought and Solution. In the Thought section, each thought should include detailed analysis, brainstorming, verification, and refinement of ideas. In the Solution section, provide the final, logical, and accurate answer, clearly derived from the exploration in the Thought section. If applicable, include the answer in \verb|\boxed{}| for closed-form results like multiple choices or mathematical solutions. Let's think step by step. 
\end{tcolorbox}
The training curves for this 35k dataset and the original 46k training dataset over 150 training steps are shown in the Figure \ref{fig:new_data}. It demonstrates that Llama3.1-8B has significantly higher learning effectiveness on the newly constructed dataset.

\subsection{Scheduling strategies of coefficient to balance \inter{} and \intra{}.}
\label{app:sch6}
We try four different scheduling strategies and show the best of them for each model. Figure \ref{fig:acc_sch} shows the dynamics of $\mu$ in the training process from the \textit{First-Inter-Later-Intra} ($\mu=1 - \Lambda$) and \textit{First-Intra-Later-Inter} ($\mu=\Lambda$). \textit{Cosin-First-Inter-Later-Intra} and \textit{Cosin-First-Intra-Later-Inter} schedule the value of $\mu$ with a cosine annealing function $\Psi$ with restarts and warm-up:
\begin{equation}
\Psi =
\begin{cases} 
\mu_{\max} \cdot \dfrac{s + 1}{w} & \text{if } s < w \\[10pt]
\mu_{\min} + \dfrac{1}{2} (\mu_{\max} - \mu_{\min}) \left(1 + \cos\left(\pi \cdot \dfrac{s^\prime}{\left\lfloor \frac{S-w}{c} \right\rfloor}\right)\right) & \text{if } s \geq w \text{ and }  s^\prime = s-w \text{ mod } \left\lfloor \frac{S-w}{c} \right\rfloor
\end{cases},
\end{equation}
where $c$ denotes the number of restart and $w$ is the warm-up step. $s$ is the current step of training and $S$ is the total step. $\mu_{\max}$ and $\mu_{\min}$ denote the specified maximum and minimum values of $\mu$. We use $c=3$, $w=30$ and $S=150$ for both strategies.

In strategy \textit{Cosin-First-Inter-Later-Intra}, we utilize $\mu=\Psi$ with $\mu_{\max}=1.0$ and $\mu_{\min}=0.4$, respectively, while in strategy \textit{Cosin-First-Intra-Later-Inter}, we utilize $\mu=1-\Psi)$ with  $\mu_{\max}=0.6$ and $\mu_{\min}=0.0$, respectively. The changes in $\mu$ under these strategies are shown in the Figure \ref{fig:step_sch}. Ultimately, based on training performance, we selected strategy \textit{Cosin-First-Inter-Later-Intra} for Qwen2.5-7B and Llama, and strategy \textit{First-Inter-Later-Intra} for Qwen2.5-1.5B.

\begin{figure}[htbp]
	\centering
	\begin{minipage}{0.43\linewidth}
		\centering
		\includegraphics[width=0.9\linewidth]{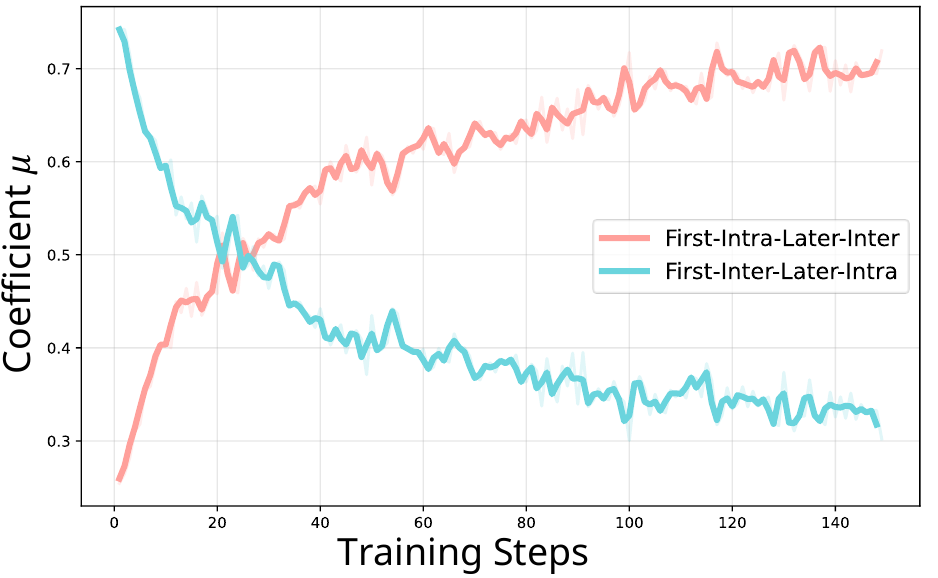}
		\caption{The changes of $\mu$ for two scheduling strategies based on accuracy during training.}
		\label{fig:acc_sch}
	\end{minipage}
    \hfill
	\begin{minipage}{0.53\linewidth}
		\centering
		\includegraphics[width=0.9\linewidth]{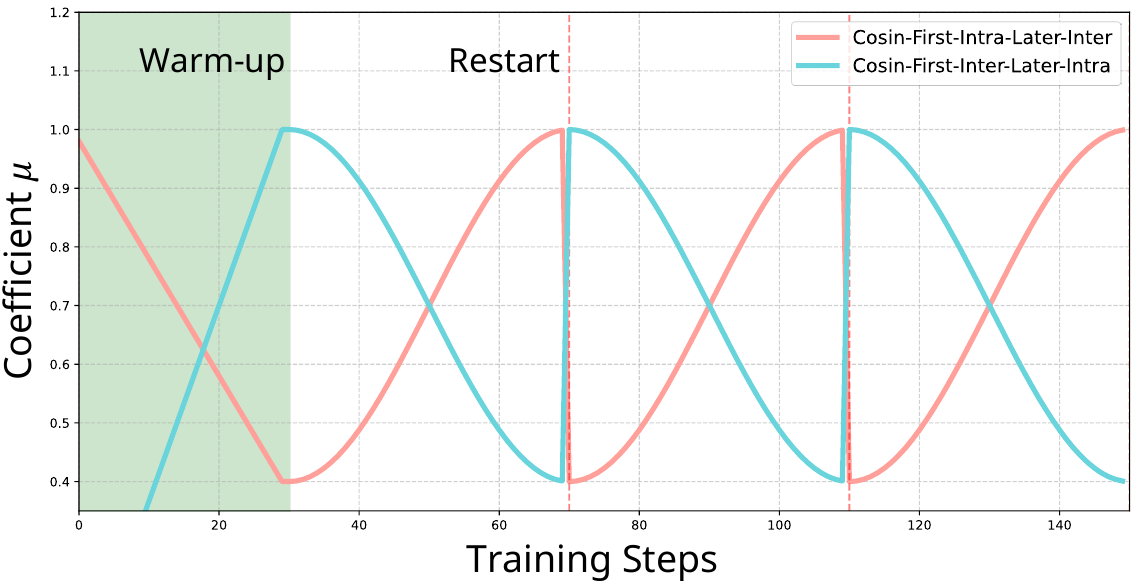}
		\caption{The changes of $\mu$ for two scheduling strategies based on training steps during training.}
		\label{fig:step_sch}
	\end{minipage}
\end{figure}

\section{Detailed derivation of Theorem \ref{thm1}}
\label{derivation}
\setcounter{theorem}{0}
\begin{theorem}[Situations with clearer advantage signal]
Suppose that condition c is based on numerical comparisons and can be derived through sorting of metrics. Further assume that the sampled response $o$ to query $q$ satisfy condition c with probability $p \in (0,1)$, and $\mathbf{E}_{o \text{~satisfy~}c}[R_o] \neq \mathbf{E}_{o \text{~not ~satisfy~}c}[R_o]$. Then, we have:
\begin{align}
\frac{|\hat{A}^{\text{inter}}_{q,o,t}|}{|\hat{A}^{\text{DR.GRPO}}_{q,o,t}|} > 1 , \text{ only when }|C^+_q|=|C^-_q| \text{ if  }|C^+_q|\text{ is a constant.  }
\end{align}
\end{theorem}
\begin{proof}[Proof of Theorem \ref{thm1}]
Given a prompt $q$, the set of all responses that satisfy condition $c$ can be denoted as $\mathcal{C}$. We use $p = \mathrm{P}(o \in\mathcal{C} | q, \theta)\in (0,1)$ to describe the probability that a response $o$ satisfying condition $c$ is provided to the prompt $q$ by an LLM with parameter $\theta$. Assuming that when condition $c$ is satisfied, the probability of the correct response is $a_+$, and when condition $c$ is not satisfied, the probability of the correct response is $a_-$. Denoting the correctness of the response $o$ to query $q$ as $R_o$, then we have:

\begin{align}
\mathbf{E}_{o\in \mathcal{C}}[R_o] = a_+ \text{ and }
\mathbf{E}_{o\notin \mathcal{C}}[R_o] = a_- ~.
\end{align}

\subsection{DR.GRPO}

Sampling a group of responses $G_q$ to the prompt $q$, the advantage $\hat{A}^{\text{DR.GRPO}}_{q,o,t}$ of a response $o$ can be calculated as: 
\begin{align}
\hat{A}^{\text{DR.GRPO}}_{q,o,t}  =  R_{o} - \text{mean}(\{R_{o^\prime }|o^\prime \in G_q\}) .
\end{align}
We use $A^{\text{DR.GRPO}}(o,c)$ to denote the \textbf{average} advantage of the responses that \textbf{satisfy} condition $c$, and utilize $\tilde{A}^{\text{DR.GRPO}}(o,c)$ to describe the average advantage of the other responses that \textbf{do not satisfy} condition $c$.

\begin{align}
    \notag A^{\text{DR.GRPO}}(o,c)&=\mathbf{E}_{o\in \mathcal{C}}[ \hat{A}^{\text{DR.GRPO}}_{q,o,t}] \\
    \notag&=\mathbf{E}_{o\in\mathcal{C}}[R_o] - \mathbf{E}_{o\in G_q}[R_o] \\
    \notag&= a_+ - [\mathrm{P}(o \in \mathcal{C} |q, \theta) \mathbf{E}_{o\in \mathcal{C}}[R_o] + \mathrm{P}(o \notin \mathcal{C} | q, \theta) \mathbf{E}_{o\notin \mathcal{C}}[R_o]] \\
    &= a_+ - pa_+ - (1-p) a_- = (a_+-a_-) (1-p) ~,\\
    \tilde{A}^{\text{DR.GRPO}}(o,c)&=(a_--a_+)p ~.
\end{align}

\subsection{Inter-group Advantage (\inter{})}
We sort the sampled responses based on the numerical value considered by condition $c$, and split them at position $k$ into two groups. Based on the symmetry of the inter-group advantage, we can denote these $k$ responses as $C^+_q$. We use $\lambda := \frac{|C_q^+|}{|G_q|}$ to simplify the notation, and denote the average inter-group advantage with $A_{\lambda}(o,c,p)$ for the responses that \textbf{satisfy} condition $c$. $\tilde{A}_{\lambda}(o,c,p)$ is utilized to represent the average inter-group advantage of those responses that \textbf{do not satisfy} condition $c$.

Then, we can compute the average reward of each group as follows.
\begin{align}
    \notag\mathbf{E}_{o\in C^+_q}[R_o] & = [\mathrm{P}(o \in \mathcal{C} | q, \theta, o\in C^+_q) \mathbf{E}_{o\in\mathcal{C}}[R_o] + \mathrm{P}(o \notin \mathcal{C} | q, \theta, o\in C^+_q) \mathbf{E}_{o\notin\mathcal{C}}[R_o]] \\
    \notag&=\left\{
                \begin{array}{ll}
                 \frac{p}{\lambda} \mathbf{E}_{o\in\mathcal{C}}[R_o] + \frac{\lambda - p}{\lambda}\mathbf{E}_{o\notin\mathcal{C}}[R_o], \text{if} ~ \lambda\geq p\\
                  \\
                \mathbf{E}_{o\in\mathcal{C}}[R_o], \text{if} ~ \lambda<p\\
                \end{array}\right.\\ 
    &=\left\{
                \begin{array}{ll}
                 \frac{p}{\lambda} a_+ + \frac{\lambda - p}{\lambda}a_-, \text{if} ~ \lambda\geq p\\
                  \\
                a_+, \text{if} ~ \lambda<p\\
                \end{array}\right.~,\\
    \notag\mathbf{E}_{o\notin C^+_q}[R_o] &= [\mathrm{P}(o \in \mathcal{C} | q, \theta, o\notin C^+_q) \mathbf{E}_{o\in\mathcal{C}}[R_o] + \mathrm{P}(o \notin \mathcal{C} | q, \theta, o\notin C^+_q) \mathbf{E}_{o\notin\mathcal{C}}R_o] \\
    \notag&=\left\{
                \begin{array}{ll}
                 \mathbf{E}_{o\notin\mathcal{C}}[R_o], \text{if} ~ \lambda\geq p\\
                  \\
                \frac{p-\lambda}{1-\lambda} \mathbf{E}_{o\in\mathcal{C}}[R_o] + \frac{1 - p}{1-\lambda}\mathbf{E}_{o\notin\mathcal{C}}[R_o], \text{if} ~ \lambda<p\\
                \end{array}\right.\\
    &=\left\{
                \begin{array}{ll}
                 a_-, \text{if} ~ \lambda\geq p\\
                  \\
                \frac{p-\lambda}{1-\lambda} a_+ + \frac{1 - p}{1-\lambda}a_-, \text{if} ~ \lambda<p\\
                \end{array}\right. ~.
\end{align}
Therefore, we can calculate the average advantages:
\begin{align}
    \notag A_{\lambda}(o,c,p) & = \mathbf{E}_{o\in \mathcal{C}}[R_o- \mathrm{P}(o \in C_q^+ | q, \theta,o\in \mathcal{C})\mathbf{E}_{o^{\prime}\notin C^+_q}[R_{o^{\prime}}] - \mathrm{P}(o \notin C_q^+ | q, \theta, o\in \mathcal{C})\mathbf{E}_{o^{\prime}\in C^+_q}[R_{o^{\prime}}]]\\
    \notag& =\mathbf{E}_{o\in \mathcal{C}}[R_o] - \left\{
                \begin{array}{ll}
                  a_-, \text{if} ~ \lambda\geq p\\
                  \\
                \frac{\lambda}{p}[\frac{p-\lambda}{1-\lambda} a_+ + \frac{1 - p}{1-\lambda}a_-], \text{if} ~ \lambda<p\\
                \end{array}\right.-\left\{
                \begin{array}{ll}
                 0, \text{if} ~ \lambda\geq p\\
                  \\
                \frac{p-\lambda}{p}a_+, \text{if} ~ \lambda<p\\
                \end{array}\right.\\
    & =\left\{
                \begin{array}{ll}
                  a_+ - a_-, \text{if} ~ \lambda\geq p\\
                  \\
                \frac{\lambda(1 - p)}{p(1-\lambda)}(a_+ - a_-), \text{if} ~ \lambda<p\\
                \end{array}\right. ~,
\end{align}
\begin{align}
    \notag \tilde{A}_{\lambda}(o,c,p) & = \mathbf{E}_{o\notin \mathcal{C}}[R_o- \mathrm{P}(o \in C_q^+ | q, \theta,o\notin \mathcal{C})\mathbf{E}_{o^{\prime}\notin C^+_q}[R_{o^{\prime}}] - \mathrm{P}(o \notin C_q^+ | q, \theta, o\notin \mathcal{C})\mathbf{E}_{o^{\prime}\in C^+_q}[R_{o^{\prime}}]]\\
    \notag& =\mathbf{E}_{o\notin \mathcal{C}}[R_o] - \left\{
                \begin{array}{ll}
                  \frac{\lambda-p}{1-p}a_-, \text{if} ~ \lambda\geq p\\
                  \\
                0, \text{if} ~ \lambda<p\\
                \end{array}\right.-\left\{
                \begin{array}{ll}
                 \frac{1-\lambda}{1-p}[\frac{p}{\lambda} a_+ + \frac{\lambda - p}{\lambda}a_-], \text{if} ~ \lambda\geq p\\
                  \\
                a_+, \text{if} ~ \lambda<p\\
                \end{array}\right.\\
    & =\left\{
                \begin{array}{ll}
                  \frac{p(1-\lambda)}{\lambda(1 - p)}(a_- - a_+), \text{if} ~ \lambda\geq p\\
                  \\
                 a_- - a_+, \text{if} ~ \lambda<p\\
                \end{array}\right. ~.
\end{align}
\subsection{Comparison}
We have the ratio between inter-group advantage and DR.GRPO:
\begin{align}
\frac{|A_{\lambda}(o,c,p)|}{|A^{\text{DR.GRPO}}(o,c)|} = 
\begin{cases} 
\frac{1}{1-p} > 1 & \text{if } \lambda\geq p\\
\frac{\lambda}{(1 - \lambda)p} & \text{if } \lambda< p
\end{cases},
\end{align}
and 
\begin{align}
\frac{|\tilde{A}_{\lambda}(o,c,p)|}{|\tilde{A}^{\text{DR.GRPO}}(o,c)|} = 
\begin{cases} 
\frac{1-\lambda}{\lambda(1 - p)} & \text{if } \lambda\geq p\\
\frac{1}{p} > 1 & \text{if } \lambda< p
\end{cases}.
\end{align}
To accentuate the impact of a specific condition on advantages, the following is required:
\begin{align}
\frac{1-\lambda}{\lambda(1 - p)}>1 \text{ if } \lambda\geq p,
\text{ and } \frac{\lambda}{(1 - \lambda)p}>1 \text{ if } \lambda< p.
\end{align}
Then we have
\begin{align}
\lambda<\frac{1}{2-p} \text{ if } \lambda\geq p,
\text{ and } \lambda>\frac{p}{1+p} \text{ if } \lambda< p .
\end{align}
If $|C^+_q|$ is a constant, $\lambda$ is also a constant. Due to $\frac{1}{2-p}>\frac{1}{2}$ and $\frac{p}{1+p} <\frac{1}{2}$, $\lambda$ needs to satisfy $\lambda\leq\frac{1}{2}$ and $\lambda\geq\frac{1}{2}$ at the same time, consequently restricting the value of $\lambda$ to 0.5. In this way, we have $\frac{|C^+_q|}{|C^+_q|+|C^-_q|}=0.5$, and finally $|C^+_q|=|C^-_q|$
\end{proof}


\end{document}